\documentclass[sigconf]{acmart}

\AtBeginDocument{%
  }

\copyrightyear{2026}
\acmYear{2026}
\setcopyright{cc}
\setcctype{by}
\acmConference[KDD '26]{Proceedings of the 32nd ACM SIGKDD Conference on Knowledge Discovery and Data Mining V.2}{August 09--13, 2026}{Jeju Island, Republic of Korea}
\acmBooktitle{Proceedings of the 32nd ACM SIGKDD Conference on Knowledge Discovery and Data Mining V.2 (KDD '26), August 09--13, 2026, Jeju Island, Republic of Korea}
\acmDOI{10.1145/3770855.3817708}
\acmISBN{979-8-4007-2259-2/2026/08}

\usepackage{hyperref}
\usepackage{url}

\usepackage{amsmath}
\usepackage{bm}
\usepackage{mathtools}
\usepackage{amsthm}
\usepackage{thm-restate}
\usepackage{thmtools} 

\usepackage{here}

\usepackage{booktabs}  
\usepackage{multirow}
\usepackage{algorithm}
\usepackage{algpseudocode}

\usepackage{float}
\usepackage{wrapfig}

\usepackage{enumitem}

\newcommand \indep{\mathop{\perp\!\!\!\!\perp}}

\DeclareMathOperator{\E}{\mathbb{E}}
\DeclareMathOperator{\V}{\mathbb{V}}
\DeclareMathOperator{\R}{\mathbb{R}}
\DeclareMathOperator{\pr}{\mathrm{P}}

\newcommand{\biX}{\textbf{\textit{X}}}
\newcommand{\bix}{\textbf{\textit{x}}}
\newcommand{\biE}{\textbf{\textit{E}}}

\newcommand{\nf}{d} 
\newcommand{\nfi}{j} 
\newcommand{\pafi}{\biX_{\mathrm{pa}(\nfi)}} 
\newcommand{\efi}{\biE_{\nfi}} 
\newcommand{\esfi}{E_{\nfi}} 
\newcommand{\cg}{G} 
\newcommand{\mfi}{m_{\nfi}(\pafi)} 
\newcommand{\vfi}{v_{\nfi}(\pafi)} 
\newcommand{\cgm}{G^{M}} 
\newcommand{\cgv}{G^{V}} 
\newcommand{\pamfi}{\biX_{\mathrm{pa}^M(\nfi)}} 
\newcommand{\pavfi}{\biX_{\mathrm{pa}^V(\nfi)}} 
\newcommand{\pammfi}{\biX_{\mathrm{pa'}^M(\nfi)}} 
\newcommand{\pavvfi}{\biX_{\mathrm{pa'}^V(\nfi)}} 

\newcommand{\mmfi}{m_{\nfi}\left(\pamfi \right)} 
\newcommand{\vvfi}{v_{\nfi}\left(\pavfi \right)} 

\newcommand{\biQ}{\textbf{\textit{Q}}} 
\newcommand{\biR}{\textbf{\textit{R}}} 
\newcommand{\biS}{\textbf{\textit{S}}} 
\newcommand{\biq}{\textbf{\textit{q}}} 
\newcommand{\bir}{\textbf{\textit{r}}} 
\newcommand{\bis}{\textbf{\textit{s}}} 
\newcommand{\biK}{\textbf{\textit{K}}} 

\newcommand{\biv}{\textbf{\textit{v}}} 

\newcommand{\biw}{\textbf{\textit{w}}}

\DeclareMathOperator{\prgs}{\mathrm{P}_{\Phi}}

\newcommand{\biA}{\mathbf{A}}

\newcommand{\biPi}{\mathbf{\Pi}}

\newcommand{\biAm}{\mathbf{A}^M}
\newcommand{\biAv}{\mathbf{A}^V}
\newcommand{\biUm}{\mathbf{U}^M}
\newcommand{\biUv}{\mathbf{U}^V}
\newcommand{\gbiUm}{\mathbf{\tilde{U}}^M}
\newcommand{\gbiUv}{\mathbf{\tilde{U}}^V}
\newcommand{\gbiPi}{{\tilde{\boldsymbol \Pi}}}

\newcommand{\gUm}{\tilde{U}^M}
\newcommand{\gUv}{\tilde{U}^V}
\newcommand{\bigumb}{\mathbf{g}}

\newcommand{\biphim}{{\boldsymbol \phi}^M}
\newcommand{\biphiv}{{\boldsymbol \phi}^V}
\newcommand{\bithetamj}{{\boldsymbol \theta}^M_{\nfi}}
\newcommand{\bithetavj}{{\boldsymbol \theta}^V_{\nfi}}

\newcommand{\bipsi}{\boldsymbol \psi}
\newcommand{\birho}{\boldsymbol \rho}

\newcommand{\biAmj}{\textbf{\textit{A}}^M_{\nfi}}
\newcommand{\biAvj}{\textbf{\textit{A}}^V_{\nfi}}

\newcommand{\biB}{\mathbf{B}}

\newcommand \myra{i}
\newcommand \myrb{i\hspace{-1pt}i}

\theoremstyle{plain}

\declaretheoremstyle[
  headpunct={:},
  headfont=\bfseries,          
  bodyfont=\itshape,           
  notefont=\normalfont,  
]{amsthmstyle}

\declaretheorem[
  style=amsthmstyle,
  name=Theorem,
  numberwithin=section
]{theorem}

\newtheorem{corollary}[theorem]{Corollary}
\theoremstyle{definition}

\theoremstyle{remark}
\newtheorem{remark}[theorem]{Remark}

\newtheorem{example}[theorem]{Example} 

\declaretheorem[
  style=amsthmstyle,
  sibling=theorem,
  name=Assumption,
  numberwithin=section
]{asmp-restate}

\declaretheorem[
  style=amsthmstyle,
  sibling=theorem,
  name=Theorem,
  numberwithin=section
]{thm-restate}

\definecolor{navyblue}{rgb}{0.0, 0.0, 0.5}
\hypersetup{
    colorlinks=true,
    allcolors=navyblue
}
\usepackage[capitalize,noabbrev]{cleveref}
\renewcommand{\eqref}[1]{(\ref{#1})}

\begin{document}

\title{Moment Matters: 
Mean and Variance Causal Graph Discovery from Heteroscedastic Observational Data}

\author{Yoichi Chikahara}
\orcid{0000-0002-9377-9046}
\affiliation{%
\department{Communication Science Laboratories}
  \institution{NTT, Inc.}
  \city{Kyoto}
  \country{Japan}
}
\email{chikahara.yoichi@gmail.com}

\begin{abstract}
Heteroscedasticity---where the variance of a variable changes with other variables---is pervasive in real data, and
elucidating why it arises from the perspective of statistical moments is crucial in scientific knowledge discovery and decision-making. 
However, standard causal discovery does not reveal which causes act on the mean versus the variance, as it returns a single moment-agnostic graph, limiting interpretability and downstream intervention design.
 We propose a Bayesian, moment-driven causal discovery framework that infers separate \textit{mean} and \textit{variance} causal graphs from observational heteroscedastic data. We first derive the identification results by establishing sufficient conditions under which these two graphs are separately identifiable. Building on this theory, we develop a variational inference method that learns a posterior distribution over both graphs, enabling principled uncertainty quantification of structural features (e.g., edges, paths, and subgraphs). 
 To address the challenges of parameter optimization in heteroscedastic models with two graph structures, we take a curvature-aware optimization approach and develop a prior incorporation technique that leverages domain knowledge on node orderings, improving sample efficiency.
 Experiments on synthetic, semi-synthetic, and real data show that our approach accurately recovers mean and variance structures and outperforms state-of-the-art baselines.

\end{abstract}




\maketitle

\newcommand\kddavailabilityurl{https://doi.org/10.5281/zenodo.20373070}
\ifdefempty{\kddavailabilityurl}{https://github.com/ychika/MV-DAG}{
\begingroup\small\noindent\raggedright\textbf{Resource Availability:}\\
The source code of this paper has been made publicly available at
\url{\kddavailabilityurl}. The corresponding GitHub repository is available at
\url{https://github.com/ychika/MV-DAG}.
\endgroup
}

\section{Introduction} \label{sec-intro}

\begin{figure}[t]
  \centering 
  \includegraphics[height=2.5cm]{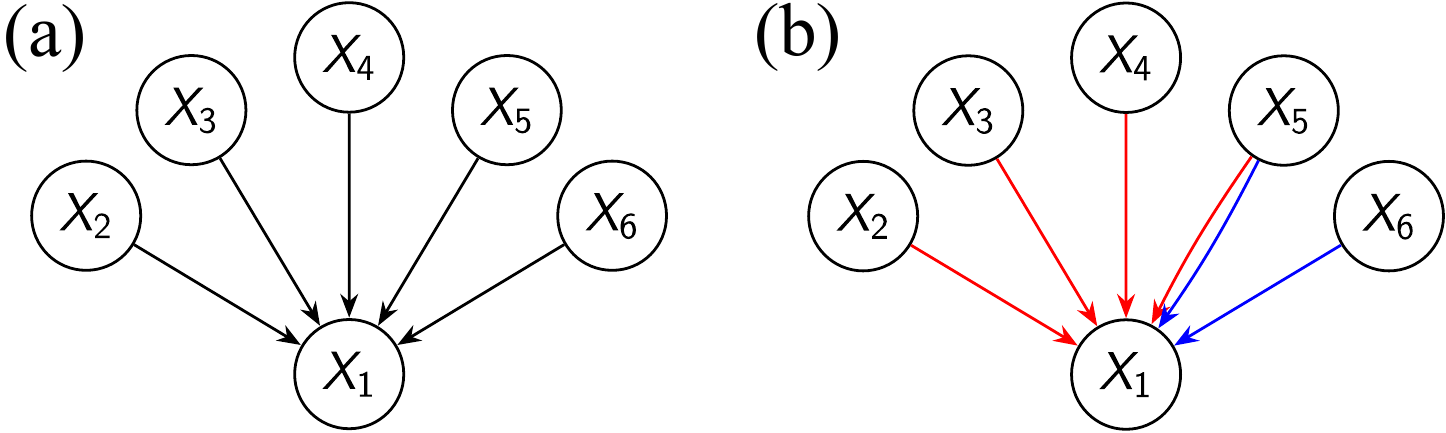}
\caption{Local protein regulatory network example represented by (a) Moment-agnostic causal graph and (b) mean and variance causal graphs (shown in red and blue edges)} 
\label{fig-ex_graph}
\end{figure}

Many scientists are actively striving to uncover the causal mechanisms that govern complex real-world phenomena. Causal discovery contributes to this goal by inferring cause-effect relationships between variables as a \textit{causal graph}, whose edge represents that one variable changes another variable values. However, such edges do not specify which statistical moments are changed by each cause, 
thereby limiting interpretability in complex systems that exhibit \textit{heteroscedasticity}, where different causes may influence the (conditional) mean and variance of each variable.
As a motivating example, below we illustrate a practical scenario for drug discovery \citep{michoel2023causal}.

\begin{example} \label{example-drug}
Consider a drug engineer who refines a compound to achieve more consistent therapeutic effects by reducing variability in a target protein’s activity across individuals.
By inferring the subgraph of underlying causal structure shown in \Cref{fig-ex_graph} (a),
the engineer selects protein $X_1$ as the downstream target and plans to intervene on its causes, $X_2, \dots, X_6$, which are pharmacologically manipulatable regulators.
 However, such standard causal graph is unsatisfactory to the engineer: Although it narrows down the candidate causes of $X_1$, 
 it does not reveal which causes influence $X_1$'s variance, 
 since such a \textit{moment-agnostic} causal graph does not distinguish the causes that affect the mean and those that influence the variance. 
 By contrast, if the engineer could obtain \textit{mean} and \textit{variance causal graphs}, which make such distinctions as shown in \Cref{fig-ex_graph} (b), they could design experiments by first focusing on achieving the desired mean of $X_1$ and subsequently varying the interventions on $X_5$ and $X_6$ to reduce $X_1$'s variance while preserving its mean. This targeted intervention strategy would greatly facilitate the drug design process for desired therapeutic outcomes.
\end{example}

 Beyond this example, the three real-world domains highlight the importance of identifying and controlling the drivers of heteroscedasticity, i.e., causes that affect the (conditional) variance of variables.
First, in \textit{systems biology}, whereas most molecules mainly shift mean expression, certain factors (e.g.,  stress-response proteins) modulate variance \citep{guilbert2020protein,boopathy2022mechanisms}; detecting such modulators deepens our understanding of the regulatory basis of cell-to-cell variability \citep{daye2012high,oyarzun2015noise}.
Second, in \textit{economics},
it is of great interest to stabilize key economic outcomes by controlling their cross-sectional variance \citep{braumoeller2006explaining}.
Third, in \textit{algorithmic fairness}, the goal is to achieve equitable outcomes  (e.g., in hiring, lending) with respect to sensitive attributes (e.g., gender, race, age, disability status) \citep{kusner2017counterfactual,wu2019pc,chikahara2021learning,chikahara2023making,zuo2024interventional}.
To promote forward-looking, fair decision-making, it is crucial to identify \textit{latent sensitive attributes}, i.e.,
features not legally designated as sensitive 
yet linked to disparities and controversial to use.
Detecting variance-level causes is helpful for this purpose, 
because as pointed out in studies in econometrics \citep{dickinson2009statistical}, 
decision rules that avoid subgroups with high outcome variance for risk aversion may induce particularly latent forms of 
statistical discrimination.

These practical real-world scenarios motivate a key research question: \textit{Can we separately identify the mean and variance causal graphs solely from observational data, collected without intervention?} 
Our answer is yes.
By extending the results on \textit{heteroscedastic noise models} (HNMs)
\citep{khemakhem2021causal,yin2024effective},
we theoretically derive the identifiability conditions for mean and variance causal graphs, whose structures may differ, depending on the moment information.

Once we have established theoretical identifiability, 
we turn to the practical challenge of inferring these two separate graphs 
from finite data.
Any two-stage point estimation approach that first infers a single moment-agnostic graph and then decomposes it into mean and variance graphs can be unreliable, especially when capturing inference uncertainty is crucial for downstream decision-making under small-sample scenarios (e.g., in medicine \citep{wiens2014study}).

This requirement raises a second research question: \textit{How can we develop a principled inference approach for quantifying the inference uncertainty of mean and variance causal graphs under data scarcity?} 
To answer this, building on the recent variational inference technique \citep{charpentier2022differentiable}, we establish a \textit{Bayesian causal discovery} approach that jointly infers the posterior distribution over the mean and variance causal graphs. 
It enables us to quantify the inference uncertainty by computing the posterior probabilities for arbitrary structural features, such as the presence of an edge, a path, or a subgraph, thereby facilitating the discovery of highly probable cause-effect relationships from finite data \citep{friedman2003being}. We demonstrate the utility of this Bayesian approach by conducting a real-world case study in \Cref{subsec-real}, where our method successfully identifies the biologically plausible variance-controlling relationship between signaling proteins from limited data.
\textbf{Our contributions} are threefold:
\begin{itemize}[leftmargin=0.5cm]
    \item \textbf{Sound definitions and theory for mean and variance graphs:} We propose a \textit{mean-variance HNM}, which is associated with mean and variance causal graphs (\Cref{subsec-problem}). We  theoretically derive their identifiability conditions in \Cref{subsec-identi}.
    \item \textbf{Principled and efficient uncertainty quantification:} We establish a variational inference framework that learns the posterior over the mean and variance causal graphs (\Cref{subsec-model,subsec-learning}). We also
    develop a prior knowledge incorporation approach for the node-ordering knowledge (\Cref{subsec-priork}). 
    \item \textbf{Strong empirical performance:} Compared to state-of-the-art baselines, our method achieves better and comparable performance in mean / variance graph inference and moment-agnostic graph estimation, respectively.
    In a real-world case study, it discovers plausible variance graph structures from limited data, demonstrating its practical utility for knowledge discovery.
\end{itemize}

\section{Background to Structural Causal Models} \label{sec-background}

A structural causal model (SCM) \citep{pearl2009causality} describes a data-generation process over a set of \textit{endogeneous variables} $\biX$, 
whose causal relationships are of interest. 
Formally, given a set of other random variables called \textit{exogenous variables} $\biE$ (e.g., noise variables)
and a set of deterministic functions $\mathcal{F}$,
an SCM defines a \textit{structural equation},
which determines the values of each $X_{\nfi} \in \biX$ 
($\nfi = 1, \dots, \nf$) as 
the output of deterministic function $f_j \in \mathcal{F}$:
\begin{align}
X_{\nfi} = f_{\nfi}(\pafi, \efi) \quad \mbox{for} \quad \nfi = 1, \dots, \nf, \label{eq-SEM}
\end{align}
where $\pafi \subseteq \biX \backslash X_{\nfi}$ and $\efi \subseteq \biE$ are subsets of endogenous and exogenous variables. 
Variables $\pafi$ directly 
affect $X_{\nfi}$'s values 
and are hence referred to as $X_{\nfi}$'s \textit{parents}.

Causal graph $\cg$ encodes these parental relationships: It has edge $X_i \rightarrow X_{\nfi}$ for $i, \nfi \in \{1, \dots , \nf\}$ and $i \neq \nfi$ if and only if $X_i \in \pafi$.
We define \textit{adjacency matrix} $\biA \in \{0, 1\}^{\nf \times \nf}$ of $\cg$ by  $A_{i, \nfi} = 1$ if $X_i \rightarrow X_{\nfi}$; otherwise, $A_{i, \nfi} = 0$.
In this paper,
we call $\cg$ a moment-agnostic causal graph,
as it does not distinguish causes by the statistical moment (e.g., mean or variance) they influence.

In general, from observational data alone,
a causal graph can be inferred only up to 
the Markov equivalence class (MEC),
i.e., the class of causal graphs 
that encode
identical conditional independence relations to joint distribution $\pr(\biX)$.
However, 
if the data are generated by an SCM belonging to 
a certain restricted functional class,
the causal graph can be uniquely identified.

The key research question of this paper is 
to go beyond such standard identifiability results for moment-agnostic causal graphs and to propose an SCM subclass that enables the separate identification of mean and variance causal graphs.

\section{Towards Moment-Driven Causal Discovery} \label{sec-problem}

As a novel class of SCMs, 
we propose mean-variance HNMs with mean and variance causal graphs in \Cref{subsec-problem}. 
We then discuss the identifiability conditions of these causal graphs in \Cref{subsec-identi}.

\subsection{Mean-Variance HNM with Mean and Variance Causal Graphs} \label{subsec-problem}

To accommodate possibly distinct mean- and variance-level causal relationships, 
we define a mean-variance HNM of each $X_{\nfi} \in \biX$ as 
\begin{align}
  X_{\nfi} = \mmfi + \vvfi \esfi \quad \mbox{for}\ \nfi = 1, \dots, \nf,\label{eq-MVHNM}
\end{align}
where $\pamfi, \pavfi \subseteq \biX \backslash X_{\nfi}$ are variable subsets,
$m_{\nfi}, v_{\nfi}$ are functions with $v_{\nfi}(\cdot) > 0$,
and $\esfi$ is a noise that has mean zero $\E[\esfi] = 0$, constant variance $\V[\esfi] > 0$ and satisfies $E_i \indep E_j$ for $i \neq j$.
With zero-mean noise $\esfi$,
the conditional mean and variance of $X_{\nfi}$ conditioned on $\pafi \coloneqq \pamfi \cup \pavfi$ are
\begin{align*}
  &\E[X_{\nfi} \mid \pafi] = m_{\nfi}(\pamfi), \\
  &\V[X_{\nfi} \mid \pafi] = \V[\esfi] \left( v_{\nfi}(\pavfi) \right)^2.
\end{align*}
We hence refer to $m_{\nfi}$ and $v_{\nfi}$ as \textit{mean} and \textit{variance functions}.

We define the mean and variance causal graphs based on parental variables $\pamfi$ and $\pavfi$:
Mean causal graph $\cgm$ has edge $X_{i} \rightarrow X_{\nfi}$ 
if and only if $X_{i} \in \pamfi$,
and variance causal graph $\cgv$ contains edge $X_{i} \rightarrow X_{\nfi}$ if and only if $X_{i} \in \pavfi$.
We define their adjacency matrices $\biAm, \biAv \in \{0, 1\}^{\nf \times \nf}$ 
by $A^M_{i, j} = 1$ if $X_i \rightarrow X_j$ in $\cgm$; otherwise, $A^M_{i, j} = 0$;
$\biAv$ is defined similarly.
We also introduce moment-agnostic causal graph $\cg$, 
whose adjacency matrix $\biA \in \{0, 1\}^{\nf \times \nf}$ 
is given by taking the logical OR as $A_{i, j} = A^M_{i, j} \lor A^V_{i, j}$.

\begin{remark}[\textbf{Connection to Existing SCM classes}] \label{remark_MVHNM}
A mean-variance HNM in Eq. \eqref{eq-MVHNM} is a generalization of the HNM \citep{xu2022inferring}, which expresses the variable values with a moment-agnostic causal graph:
\begin{equation}
  X_{\nfi} = \mfi + \vfi \esfi \quad \mbox{for}\ \nfi = 1, \dots, \nf. \label{eq-HNM}
\end{equation}
An HNM is an extension of additive noise models (ANMs) \citep{shimizu2006linear,hoyer2008nonlinear}, which assume homoscedasticity with constant variance function $v_{\nfi}(\cdot)$ in \eqref{eq-HNM}.
Both ANMs and HNMs assume additive noise structure and hence cannot address data distortion (e.g., log transform) 
unlike post-nonlinear models (PNLs) \citep{zhang2009identifiability}. However, unlike PNLs,
HNMs can capture heteroscedasticity,
which is pervasive in complex real-world data, such as gene expression data \citep{imoto2003bayesian}.

The mean-variance HNM has the same expressive power as the original HNM.
When the input sets are identical ($\pamfi = \pavfi = \pafi$), our mean-variance HNM reduces to the original HNM. Conversely, any instance of a mean-variance HNM can be rewritten 
as the original HNM via input masking: 
The mean and variance functions in the former 
can be reformulated as
$m_{\nfi}(\biS_{\mathrm{pa}^M(\nfi)}(\pafi))$ and $v_{\nfi}(\biS_{\mathrm{pa}^V(\nfi)}(\pafi))$,
where $\biS_{\mathcal{I}} \colon$ $\pafi$ $\rightarrow \biX_{\mathcal{I}}$ is a masking function that selects the variables indexed by $\mathcal{I} \subset \{1, \dots, \nf\}$ from parents $\pafi$ in moment-agnostic graph $\cg$.
Thus, two HNMs express the same class of functions. 
\end{remark}

Such functional equivalence between two HNMs implies that,
to identify mean and variance causal graphs $\cgm$ and $\cgv$,
we must identify not only the moment-agnostic causal graph $\cg$ of the original HNM
but also the variable selection mechanism for $\cgm$ and $\cgv$.
Based on this implication, we derive the identifiability conditions by building on the existing results for the original HNM \citep{khemakhem2021causal,yin2024effective}. 

\subsection{Identifiability Results} \label{subsec-identi}

Our results rely on four assumptions (See \Cref{sec-asmp} for details). 
The first two coincide with those for the original HNM \citep{yin2024effective}:
\begin{restatable}[Causal sufficiency]{asmp-restate}{asmpCsuff} \label{asmp-csuff}
  Exogenous noises satisfy
  $E_{i} \indep E_{\nfi}$ for any $i, \nfi \in \{1, \dots, \nf\}$.
\end{restatable}
\begin{restatable}[Causal minimality]{asmp-restate}{asmpCm} \label{asmp-cm}
  $\pr(\biX)$ satisfies the \textit{causal minimality condition} with respect to moment-agnostic causal graph $\cg$.
\end{restatable}
The last two ensure the acyclicity of moment-agnostic graph $\cg$:
\begin{restatable}{asmp-restate}{asmpDag} \label{asmp-dag}
  Mean and variance graphs $\cgm$ and $\cgv$ are directed acyclic graphs (DAGs).
\end{restatable}
\begin{restatable}[Shared permutation condition]{asmp-restate}{asmpOrder} \label{asmp-order}
  There exists an identical permutation (a.k.a., \textit{topological ordering}) of mean and variance causal graphs $\cgm$ and $\cgv$.
\end{restatable}
\setcounter{theorem}{4} 
Permutation $\pi: \{1, \dots, \nf\} \rightarrow \{1, \dots, \nf\}$ 
is a node ordering such that 
no node has a directed path to any node that precedes it in the order; that is, 
if $\pi(i) < \pi(j)$, then $X_{j}$ cannot have a directed path to $X_{i}$.
It is \textbf{not} unique because a DAG admits multiple permutations.

For example, 
Assumptions \ref{asmp-dag} and \ref{asmp-order} 
indicate that
if mean graph $\cgm$ contains $X_i \rightarrow X_j$, 
then variance graph $\cgv$ cannot have the reverse edge $X_i \leftarrow X_j$ and must either 
(\myra) also include $X_i \rightarrow X_j$ 
or (\myrb) contain no edge,
implying that the structural differences between $\cgm$ and $\cgv$ 
lie in the presence or absence of each edge.

Thus, Assumptions \ref{asmp-dag} and \ref{asmp-order} 
ensure that the \textit{union} of $\cgm$ and $\cgv$ (i.e., moment-agnostic graph $\cg$) is a DAG (see \Cref{cor-DAG}).
Since there are no identifiability results for \textit{cyclic} HNMs \citep{strobl2023identifying,yin2024effective},
this DAG condition (and hence Assumptions \ref{asmp-dag} and \ref{asmp-order}) follows a standard identifiability route in the HNM literature.

Based on these assumptions, 
we derive the \textbf{sufficient} (\textbf{but not necessary}) identifiability conditions for $\cgm$ and $\cgv$:
\begin{restatable}{theorem}{identifiability} \label{th2} 
  Under Assumptions \ref{asmp-csuff}, \ref{asmp-cm}, \ref{asmp-dag}, and \ref{asmp-order},
  mean and variance causal graphs
  $\cgm$ and $\cgv$ are identifiable
  from observational distribution $\pr(\biX)$
  if for $\nfi = 1, \dots, \nf$,
  (A) $m_{\nfi}$ is a nonlinear function, 
  (B) $v_{\nfi}$ is a piecewise function, \textbf{but not a constant function}, and 
  (C) $E_{\nfi}$ is a Gaussian noise.
  \end{restatable}
\begin{proof}[Proof sketch]
Our proof in \Cref{sec-proof} takes two steps. 
First, we show that Conditions (A), (B), and (C) ensure identifiability of the moment-agnostic causal graph $\cg$ for our mean-variance HNM. 
We prove that these conditions are sufficient
to rule out the unidentifiable cases (e.g., $m_j$ is linear or $v_j$ is the inverse of a two-order polynomial),
which are established by \citet{khemakhem2021causal} for the original HNM but apply to our mean-variance HNM under Assumptions \ref{asmp-dag} and \ref{asmp-order}  (\Cref{subsec-moment-agnostic}). 
Second, for $\nfi \in \{1,\dots,\nf\}$, we prove that the parent set $\pafi$ in $\cg$ can be decomposed into subsets $\pamfi$ and $\pavfi$ under Gaussian noise $E_{\nfi}$
by showing that $\pr(\biX)$ cannot coincide under different choices of $\pamfi$ and $\pavfi$, due to the non-constancy of $m_{\nfi}$ and $v_{\nfi}$ (\Cref{subsec-mean-var-graph}).
\end{proof}
Our conditions in \Cref{th2}
are explicitly stated 
in terms of the forms of functions $m_{\nfi}$ and $v_{\nfi}$ and the distribution of noise $E_{\nfi}$,
which \textbf{can be easily enforced} via likelihood modeling (\Cref{subsec-model}) to ensure the model remains within an identifiable class.

\begin{remark}[\textbf{Role of Gaussianity}] \label{remark}
  Gaussian noise condition (C) is crucial to enable \textbf{moment-based separation of causes}. 
  Even with non-Gaussian noise $E_{\nfi}$, the causes $\pavfi$ in graph $\cgv$ (in Eq. \eqref{eq-MVHNM}) are guaranteed \textbf{not} to affect the mean of $X_{\nfi}$,
  but they may influence not only variance but also its higher-order moments (e.g., skewness, kurtosis).
  For moment-agnostic graph inference,
  the Gaussian condition is \textbf{standard},
  as existing HNM-based methods cannot handle multivariate non-Gaussian HNMs \citep{duong2023heteroscedastic,yin2024effective,kikuchi2022differentiable, immer2023identifiability,strobl2023identifying,tran2024robust}. 
  See \Cref{fig-nonGaussian} for the results on non-Gaussian synthetic datasets.
\end{remark}
\begin{remark}[Non-constancy] \label{remark}
  Compared with \citet{yin2024effective}, our non-constancy Condition (B) on variance function $v_{\nfi}$ is stronger.
  However, excluding constant functions is standard for the identifiability, as done for ANMs \citep[Proposition 17]{peters2014causal}.  
  In our setup, it is crucial to distinguish the causes based on the mean and variance. 
\end{remark}

\section{Proposed Method} \label{sec-method}

\subsection{Model Formulations} \label{subsec-model}

We develop a Bayesian approach for inferring the posterior distribution over adjacency matrices $\biAm$ and $\biAv$ of the mean and variance causal graphs, conditioned on observational data $\mathcal{D}$:
\begin{align}
  \pr(\biAm, \biAv \mid \mathcal{D}) \propto \pr(\biAm, \biAv) 
   \pr(\mathcal{D} \mid \biAm, \biAv), 
  \label{eq-posterior}
\end{align}
where $\pr(\biAm, \biAv)$ is a prior, and 
\begin{align}
   \pr(\mathcal{D} \mid \biAm, \biAv) = \int \pr(\Theta \mid \biAm, \biAv) \pr(\mathcal{D} \mid \Theta, \biAm, \biAv) \mathrm{d}\Theta
  \label{eq-integral}
\end{align} 
is a marginal likelihood obtained by integrating out the likelihood model parameters $\Theta$. 
Unfortunately, tractable computation of the integral in Eq. \eqref{eq-integral}
requires restrictive parameterizations,
such as linear Gaussian models,
which may lead to inaccurate inferences when the model assumptions are violated.

For this reason,
we approximate the posterior 
with a variational distribution (with parameters $\Phi$) 
as $\pr(\biAm, \biAv \mid \mathcal{D})$ $\approx$ $\prgs(\biAm, \biAv)$.
Below we present this DAG distribution and the likelihood.

\subsubsection{DAG Distribution Model} \label{subsubsec-dm}

We formulate a theoretically grounded model $\prgs(\biAm, \biAv)$ 
by ensuring that $\biAm$ and $\biAv$ are the DAGs with a shared permutation (i.e., Assumptions \ref{asmp-dag} and 
\ref{asmp-order}).

To guarantee that adjacency matrices $\biAm$ and $\biAv$ represent such DAGs, 
we build on a well-known result in graph theory
and decompose them using a shared permutation matrix $\biPi$:
\begin{align}
  \biAm = \biPi^{\top} \biUm \biPi,\quad
  \biAv = \biPi^{\top} \biUv \biPi, \label{eq-DAG} 
\end{align}
where  $\biUm,\biUv$  $\in \{0, 1\}^{\nf \times \nf}$ are upper-triangular matrices,
and $\biPi \in \{0, 1\}^{\nf \times \nf}$ is a permutation matrix,
in which every row and every column has exactly one element of 1 
with all others 0.
Eq. \eqref{eq-DAG} simply states that 
permuting the elements in an upper-triangular matrix leads to a DAG adjacency matrix,
e.g., $U^M_{\pi(i), \pi(j)} = A^M_{i, j}$,
where $\pi$ is a permutation that takes $\pi(i) = j$ if and only if $\Pi_{i,j} = 1$.

Using the DAG decomposition in \eqref{eq-DAG}, 
we consider a variational distribution based on the following factorization:
\begin{align}
\pr(\biAm, \biAv) = \sum_{\biUm, \biUv, \biPi} \pr(\biUm) \pr(\biUv) \pr(\biPi),  
\nonumber
\end{align}
where the summation is over all upper-triangular matrices $\biUm$ and $\biUv$ and permutation matrices $\biPi$ that are compatible with a DAG pair, $\biAm$ and $\biAv$.
Unfortunately, the exact computation of this summation is intractable,
as the number of valid permutation matrices $\biPi$ grows exponentially with the number of nodes $\nf$. 

For this reason, we take a sampling-based learning approach 
by generalizing the idea of differentiable DAG sampling (DDS) \citep{charpentier2022differentiable} to our setup 
with two DAG adjacency matrices $\biAm$ and $\biAv$.
To enable backpropagation in a backward pass, 
we approximately compute the gradient of the sampling operations 
for binary matrices $\biUm$, $\biUv$, and $\biPi$ in \eqref{eq-DAG} 
by leveraging their continuous relaxation $\gbiUm, \gbiUv \in [0, 1]^{\nf \times \nf}$ and $\gbiPi \in  \R^{\nf \times \nf}$.
To obtain such continuous matrices $\gbiUm$ and $\gbiUv$, 
we employ the \textit{Gumbel-Softmax trick} \citep{jang2017categorical}, 
which allows differentiable approximation 
using i.i.d. standard Gumbel noises $g_0^M, g_0^V, g_1^M, g_1^V \sim \mathrm{Gumbel}(0)$:
\begin{align*}
  &\gUm_{i, j} = \frac{\mathrm{e}^{(\mathrm{log} \phi^M_{i, j} + g_1^M ) / \tau^M}}{\mathrm{e}^{(\mathrm{log} \phi^M_{i, j} + g_0^M ) / \tau^M} + \mathrm{e}^{(\mathrm{log} (1 - \phi^M_{i, j}) + g_1^M ) / \tau^M}},\\ 
  &\gUv_{i, j} = \frac{\mathrm{e}^{(\mathrm{log} \phi^V_{i, j} + g_1^V ) / \tau^V}}{\mathrm{e}^{(\mathrm{log} \phi^V_{i, j} + g_0^V ) / \tau^V} + \mathrm{e}^{(\mathrm{log} (1 - \phi^V_{i, j}) + g_1^V ) / \tau^V}}, 
\end{align*}
where 
$\phi^M_{i, j}, \phi^V_{i, j} \in [0, 1]$ denote the Bernoulli distribution parameters, 
and $\tau^M, \tau^V \geq 0 $ are the hyperparameters that control the smoothness of the distribution.
To compute $\gbiPi$, 
we directly adopt the DDS's sampling scheme,
which approximates the sorting order of perturbed log probabilities: 
\begin{align} 
  \gbiPi = \mathrm{SoftSort}(\mathrm{log} \bipsi + \bigumb),\label{eq-gbiPi}
\end{align}
where $\bipsi \in \R^\nf$ denotes the categorical distribution parameters,
$\bigumb \in \R^\nf$ is a vector of i.i.d. standard Gumbel noise,
and $\mathrm{SoftSort}$ is a differentiable approximation to a sorting operation, 
which returns a continuous relaxation of permutation matrix representing the sorted order of the input elements.  
We detail the formulation of $\mathrm{SoftSort}$ function in \Cref{subsec-dperm}. 

In a forward pass, we sample binary matrices $\biUm$, $\biUv$, $\biPi$ $\in \{0, 1\}^{\nf \times \nf}$
by applying the argmax operator on $\gbiUm$ and $\gbiUv$
(e.g.,  $U^M_{i, j}$ $=$ $\arg \max [1 - \gUm_{i, j}, \gUm_{i, j}]$)
and a row-wise argmax operation on $\gbiPi \in \R^{\nf \times \nf}$.
By plugging the sample of $\biUm$, $\biUv$, and $\biPi$ 
into Eq. \eqref{eq-DAG},
we sample DAG adjacency matrices $\biAm, \biAv \sim \pr_{\Phi}(\biAm, \biAv)$,
where $\Phi = \{\biphim, \biphiv, \bipsi\}$ is a set of parameters. 

\subsubsection{Likelihood Model} \label{subsubsec-ll}

Using DAG sample $\biAm, \biAv \sim \pr_{\Phi}(\biAm, \biAv)$,
we model likelihood
$\pr_{\Theta}(\mathcal{D} | \biAm, \biAv)$
based on mean-variance HNMs.

For each $\nfi = 1, \dots, \nf$,
we parameterize mean and variance functions $m_{\nfi}$ and $v_{\nfi}$ 
using multi-layer perceptrons (MLPs) 
with parameters $\bithetamj, \bithetavj \in \Theta$.
To represent the (possibly different) inputs of these MLPs,
we apply masking to $\biX$ using the $\nfi$-th column vectors in $\biAm$ and $\biAv$, 
denoted by $\biAmj, \biAvj \in \{0, 1\}^{\nf}$,
leading to the following structural equation parameterization:
\begin{align} 
  X_{\nfi} = m_{\nfi}(\biAmj \odot \biX; \bithetamj) + v_{\nfi}(\biAvj \odot \biX; \bithetavj) E_{\nfi}\ \mbox{for}\ \nfi = 1, \dots, \nf,  \label{eq-MLP}
\end{align}
where $\odot$ denotes the Hadamard product.
By assuming, without loss of generality, that 
noise $E_{\nfi}$ follows a standard Gaussian distribution $\mathcal{N}(0, 1)$,
we parameterize the conditional distribution of $X_{\nfi}$ in \eqref{eq-MLP} as multivariate Gaussian $\mathcal{N}(m_{\nfi}(\biAmj \odot \biX; \bithetamj), 
( v_{\nfi}(\biAvj \odot \biX; \bithetavj) )^2 )$.

To satisfy Conditions (A) and (B) in \Cref{th2},
we formulate all MLPs in our experiments
using leaky rectified linear unit (ReLU) activations, 
which are piecewise nonlinear and cannot be a constant.
To ensure that the variance function value is strictly positive 
(i.e., $v_{\nfi}(\cdot) > 0$),
we model $\mathrm{log}\ v_{\nfi}(\cdot)$ using an MLP
and apply the exponential function to it as a final transformation.

\subsection{Model Parameter Learning Overview} \label{subsec-learning}

To obtain approximated posterior 
$\prgs(\biAm, \biAv) \simeq \pr(\biAm, \biAv \mid \mathcal{D})$,
we learn parameters $\Phi = \{\biphim, \biphiv, \bipsi\}$, and $\Theta = \{(\bithetamj, \bithetavj)_{\nfi=1}^{\nf}\}$,
such that 
the Kullback Leibler (KL) divergence between variational distribution $\prgs(\biAm, \biAv)$ and true posterior $\pr(\biAm, \biAv \mid \mathcal{D})$ is minimized. 
Since minimizing this KL divergence is equivalent to maximizing the evidence lower bound (ELBO), we aim to maximize the following ELBO-based objective function:
\begin{align} 
	\underset{\Phi,\Theta}{\text{max}}
 \quad &\E_{\biAm, \biAv \sim \prgs}\left[\mathrm{log}\ \pr_{\Theta}(\mathcal{D} \mid \biAm, \biAv)\right] - \lambda\Omega_{\Phi, \Theta},
  \label{eq-plogl}
\end{align}
where $\mathrm{log} \pr_{\Theta}(\mathcal{D}| \biAm, \biAv)$ is a log likelihood, $\lambda \geq 0$ is a parameter for the following sparsity-inducing regularizer:
\begin{align}
  \begin{aligned}
  \Omega_{\Phi, \Theta} &= \lambda_{\Phi} \mathrm{KL}\left(\prgs(\biUm, \biUv) \mid\mid \pr(\biUm, \biUv)\right) \\
  &+ \lambda_{\Theta^M} \sum_{\nfi=1}^{\nf} \|\bithetamj\|_2^2 + \lambda_{\Theta^V} \sum_{\nfi=1}^{\nf} \|\bithetavj\|_2^2, 
  \end{aligned}
  \label{eq-reg}
\end{align}
where $\lambda_{\Phi}, \lambda_{\Theta^M}, \lambda_{\Theta^V} \geq 0$ are hyperparameters. 
Following DDS \citep{charpentier2022differentiable},
we do not put any prior on permutation matrix $\biPi$ in Eq. \eqref{eq-reg}
to retain tractable computation.  
We encourage edge sparsity 
by setting prior probabilities $\pr(\biUm, \biUv)$ to small values.

To solve the optimization problem in \eqref{eq-plogl},
we repeat three steps.
First, we sample DAG adjacency matrices $\biAm, \biAv \sim \prgs(\biAm, \biAv)$.
Then we use this single DAG sample to approximate the expectation in Eq. \eqref{eq-plogl} 
as the Gaussian log likelihood $\mathrm{log} \pr_{\Theta}(\mathcal{D}| \biAm, \biAv)$ for observational dataset $\mathcal{D} = \{\bix_1, \dots, \bix_n\}$:
\begin{align}
  -\frac{1}{n} \sum_{i=1}^n \sum_{\nfi = 1}^{\nf} \bigg(  
    \frac{(x_{i, j} - m_{\nfi}(\bix_{i, \mathrm{pa}^M(j)}; \bithetamj))^2}{2 (v_{\nfi}(\bix_{i, \mathrm{pa}^V(j)}; \bithetavj))^2} +
    \mathrm{log} v_{\nfi}(\bix_{i, \mathrm{pa}^V(j)}; \bithetavj)
    \bigg) \nonumber
  \end{align}
where $\bix_{i, \mathrm{pa}^M(j)} \coloneqq \biAmj \odot \bix_i$ and $\bix_{i, \mathrm{pa}^V(j)} \coloneqq \biAvj \odot \bix_i$ are the masked input vectors obtained from observation $\bix_i$. 
Finally, we perform a gradient-based update for $\Phi$ and $\Theta$ based on Eq. \eqref{eq-plogl}.

\subsection{Challenges in Complex Heteroscedasticity and Two-Graph Inference} \label{subsec-difficulty}

\begin{algorithm}[t]
\caption{Parameter Learning for Mean and Variance Graphs}
\label{alg:mvhnm}
\begin{algorithmic}[1]
\Require Observational dataset $\mathcal{D}=\{\bix_1, \dots, \bix_n\}$; parameters $\Phi,\Theta$; node-ordering constraints $\mathcal O$
\State Initialize all $\Phi = \{\biphim, \biphiv, \bipsi\}$ and $\Theta = \{(\bithetamj, \bithetavj)_{\nfi=1}^{\nf}\}$
\While{Stopping criterion not met}
  \While{Updates for $\Theta \setminus \{(\bithetavj)_{j=1}^{\nf}\}$ and $\Phi \setminus \{\biphiv\}$  not converged}
  \State Sample $\biAm,\biAv \sim \pr_{\Phi}(\biAm, \biAv)$
  \State Compute ELBO-based objective function in Eq. \eqref{eq-plogl}
  \State Update $\Theta \setminus \{(\bithetavj)_{j=1}^{\nf}\}$ and $\Phi \setminus \{\biphiv\}$ using scaled gradient
  \If{node-ordering constraints $\mathcal{O}$ are provided}
     \State Project updated $\bipsi$ onto the feasible set by Eq. \eqref{eq-priork}
  \EndIf
  \EndWhile
  \While{Updates for variance-related parameters $\{(\bithetavj)_{j=1}^{\nf}\}$ and $\{\biphiv\}$ not converged}
  \State Sample $\biAm,\biAv \sim \pr_{\Phi}(\biAm, \biAv)$
  \State Compute ELBO-based objective function in Eq. \eqref{eq-plogl}
  \State Update $\{(\bithetavj)_{j=1}^{\nf}\}$ and $\{\biphiv\}$ using standard gradient
  \EndWhile
\EndWhile
\State \Return Learned parameters $\Phi$ and $\Theta$
\end{algorithmic}
\end{algorithm}

The optimization of the objective in Eq. \eqref{eq-plogl} faces two challenges.

\textbf{1. Difficulty in Maximum Likelihood Estimation (MLE):} 
Gradient-based updates for likelihood parameters $\Theta$ has two issues. 
First, maximizing the likelihood with respect to $\Theta$
is \textit{ill-posed}
in the sense that it can be arbitrarily increased by inflating the variance function value 
rather than by reducing the residual errors \citep{seitzerpitfalls}.
Second,
gradient-based optimization slows around the data points with high true variance
because the gradients with respect to $m_{\nfi}$ and $v_{\nfi}$ scale inversely with variance function $v_{\nfi}(\bix_{i, \mathrm{pa}^V(j)}; \bithetavj)$ \citep{skafte2019reliable}.

We overcome these two issues by building on recent advances in heteroscedastic noise regression \citep{stirn2023faithful,wong2024understanding}.
To address the ill-posedness of MLE,
we impose $l_2$ regularization on $\Theta$, as shown in Eq. \eqref{eq-reg}.
To resolve the optimization inefficiency,
we alternately take two steps.
First, we update all the parameters except for the variance-related ones,  
i.e., $\Theta \setminus \{(\bithetavj)_{j=1}^{\nf}\}$ and $\Phi \setminus \biphiv$,
using the gradient of mean squared error (MSE), $\frac{1}{n} \sum_{i,j} (x_{i, j} - m_{\nfi}(\bix_{i, \mathrm{pa}^M(j)}; \bithetamj))^2$,
which is equivalent to scaling the standard gradient by $( v_{\nfi}(\bix_{i, \mathrm{pa}^V(j)}; \bithetavj) )^2$. This scaling approximates
a computationally demanding second-order Newton step with an inverse Jacobian,
thus enabling an efficient, curvature-aware optimization.
Second, we update variance-related parameters $\{(\bithetavj)_{j=1}^{\nf}\}$ and $\biphiv$ using the standard gradient.
\Cref{alg:mvhnm} summarizes the overall procedure.
The total time complexity is dominated by likelihood evaluation and is identical to the existing likelihood-based methods \citep{yin2024effective}.

\textbf{2. Large Parameter Search Space:} 
Fitting our DAG distribution model parameter $\Phi$ for the two causal DAGs is challenging 
because it is inevitably more high-dimensional than the existing models for a single (moment-agnostic) causal DAG.
Although we effectively reduce the number of parameters
by sharing permutation distribution parameter $\bipsi \subset \Phi$ (\Cref{subsubsec-dm}),
its estimation might remain difficult under small sample-size scenarios.
To tackle this challenge, below 
we present an approach that 
reduces the search space of $\bipsi$ 
by incorporating prior knowledge about the node orderings.

\subsection{Prior Knowledge Incorporation} \label{subsec-priork}

Incorporating domain knowledge about the ground-truth graph has been a key strategy for addressing small-sample setups \citep{niinimaki2016structure,oyen2017order,li2018bayesian,hasan2022kcrl}. 
A common type of such knowledge is the absence or presence of specific edges.
Edge absence knowledge can be easily incorporated by reducing the prior probabilities $\pr(\biUm, \biUv)$ in Eq. \eqref{eq-reg} for the corresponding edges.
By contrast, edge presence knowledge is more delicate in our setting because forcing an edge to appear in the sampled graphs $\biAm$ and $\biAv$ must preserve the DAG constraints.
This naturally leads to pairwise node-ordering constraints, which specify that one variable should precede another.
Moreover, such ordering knowledge is often available in practice \citep{ban2024differentiable};
for example, in gene regulatory network inference, we may know that $X_i$ is an upstream regulator of gene $X_{\nfi}$.
For these two reasons, below we focus on incorporating pairwise node-ordering knowledge.

Suppose that we have access to a set of pairwise node-ordering constraints $\mathcal{O} = \{\pi(i) < \pi(\nfi) | (i, \nfi) \in \mathcal{S}\}$ 
for variable index pairs $\mathcal{S} \subseteq \{(i, \nfi) | i, \nfi \in \{1, \dots, \nf\}\}$,
where $\pi(i) < \pi(\nfi)$ is a node-ordering constraint that variable $X_{i}$ precedes $X_{\nfi}$ in permutation $\pi$.

\begin{figure*}[t]
  \includegraphics[height=4.7cm]{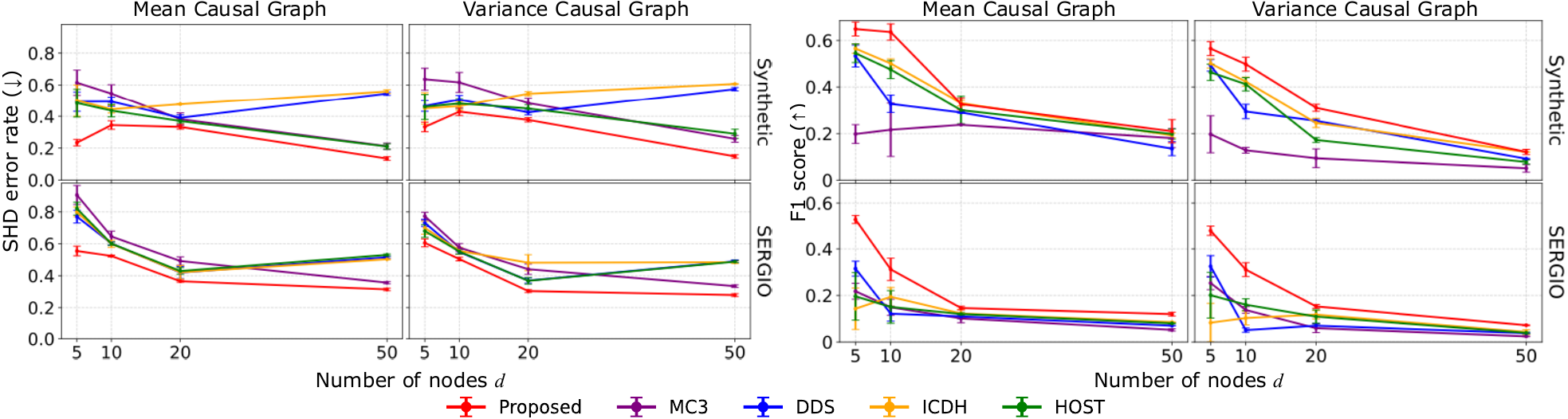}
\centering 
\caption{Mean and variance causal graph inference performance on synthetic and SERGIO datasets with sample size $n=500$. Achieving \textbf{both} lower SHD rate (left) and higher F1 score (right) is better.} 
\label{fig-diffHNMNN}
\end{figure*}

\begin{figure}[t]
  \includegraphics[height=9.2cm]{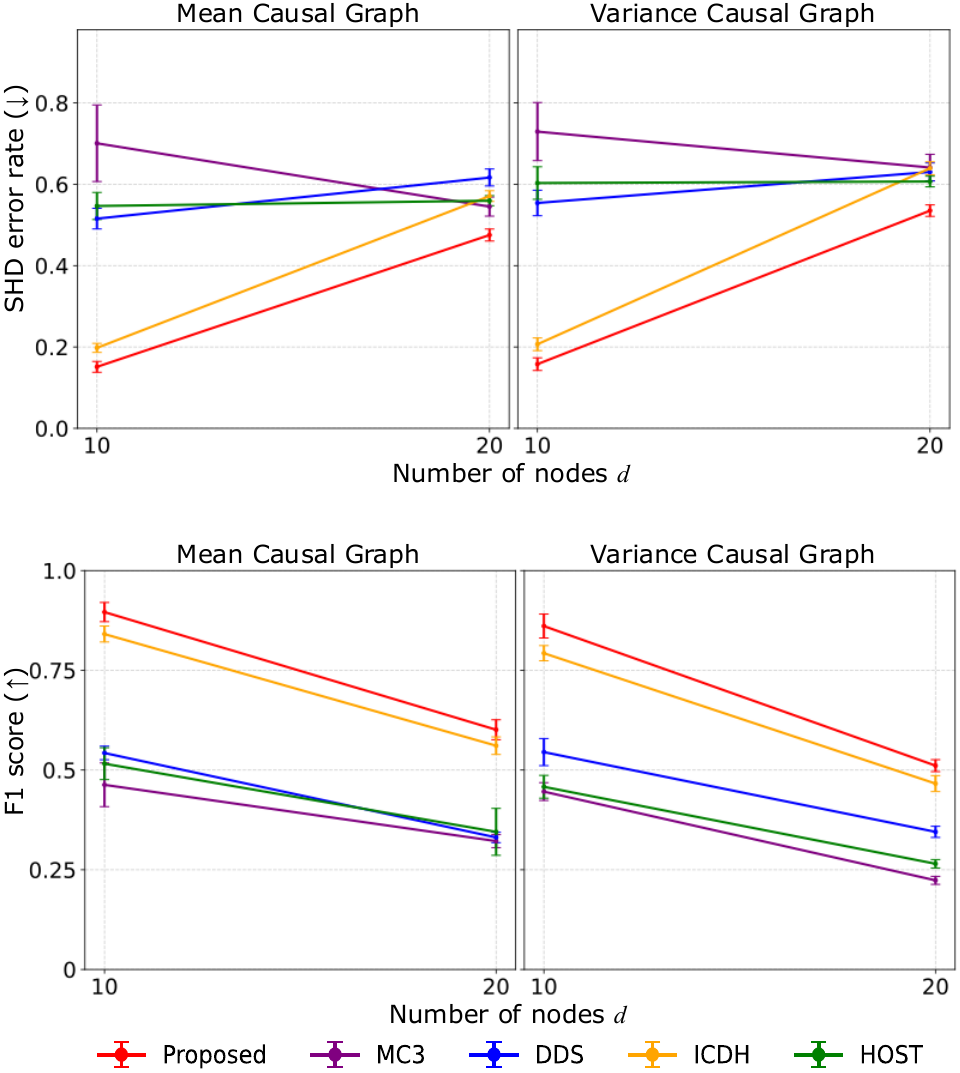}
\centering 
\caption{Mean and variance causal graph inference performance on synthetic datasets ($n=500$) under dense graphs} 
\label{fig-dense}
\end{figure}

Then we can easily incorporate such node-ordering constraints, 
thanks to the sorting-based sampling scheme in \eqref{eq-gbiPi}.
Since this scheme samples the continuous relaxation of the permutation matrix 
by computing a sorting order for parameter vector $\bipsi \in \R^{\nf}$,
we can impose the node-ordering constraints in $\mathcal{O}$ 
by forcing the parameter values 
to satisfy the following inequality constraint set:
\begin{align}
  \mathcal{I}(\bipsi) = \{\psi_i + c_{i, \nfi} \leq \psi_{\nfi} \mid  (i, \nfi) \in \mathcal{S}\}, \label{eq-ineq-constr}
\end{align}
where $c_{i, \nfi} > 0$ is a hyperparameter that ensures $\psi_i < \psi_{\nfi}$.
To satisfy these constraints, 
we project the updated parameter values $\bipsi'$ 
onto the feasible set
by solving the constrained least squares:
\begin{align}
  \bipsi_{\mathrm{new}} = \underset{\birho \in \mathcal{I}(\birho)}{\text{arg min}} \| \birho - \bipsi'\|^2_2, \label{eq-priork}
\end{align}
which can be efficiently solved with a quadratic programming (QP) solver.
The constraints in \eqref{eq-priork} are
soft constraints,
since they might not hold in sampled permutations
due to the perturbation by Gumbel noise in Eq. \eqref{eq-gbiPi}.
Hyperparameter $c_{i, \nfi}$ strikes a trade-off between the robustness to this noise and possible parameter values;
we set $c_{i, \nfi} = 1.5$ for all $i, \nfi \in \{1, \dots, \nf\}$ in our experiments.

\section{Experiments} \label{sec-exp}

\textbf{Baselines:} We compare our method with four baselines.
(\myra) Two Bayesian methods based on ANM likelihoods:  
Metropolis-coupled Markov chain Monte Carlo (\textbf{MC3}) \citep{giudici2003improving}, a traditional posterior-sampling approach using linear-Gaussian models; and  
a variational inference method, \textbf{DDS} \citep{charpentier2022differentiable}. 
(\myrb) Two point-estimation methods based on HNM likelihoods, identifiable causal discovery under heteroscedastic data (\textbf{ICDH}) \citep{yin2024effective}
and heteroscedastic causal structure learning (\textbf{HOST}) \citep{duong2023heteroscedastic}.
Following \citet{yin2024effective},
we exclude VarSort \citep{reisach2021beware}, 
as its performance is highly unstable under heteroscedastic noise, due to the reliance on marginal variances
(see \Cref{subsec-varsort}).

To make a fair comparison, 
we test \textbf{both} moment-specific and moment-agnostic graph inference performance of all methods, 
mirroring the spirit of \citet{yin2024effective}, 
who evaluated their HNM-based method \textbf{ICDH} 
on both ANM and HNM datasets.

\subsection{Simulated Data Experiments} \label{subsec-synth}

We evaluate our method on synthetic and semi-synthetic datasets.

\subsubsection{Mean and Variance Graph Inference} \label{subsubsec-synth-mv}

\textbf{Data and Causal Graphs:}
We sample synthetic datasets from mean-variance HNMs,  
where both the mean and variance functions are modeled by MLPs with randomly initialized parameters.
We randomly draw ground-truth mean and variance causal graphs from the Erd\H{o}s--R\'enyi (ER) models,
where the expected number of edges per node is set to 1 or 2 (sparse graphs; 1 for cases with $\nf = 5$) or 4 (dense graphs).
We also use semi-synthetic datasets generated from SERGIO \citep{dibaeinia2020sergio}, 
a gene expression simulator that produces \textit{realistic} 
single-cell transcriptomic data using the parameters tuned to real data.
Their ground truth causal graphs are sampled from the scale-free (SF) models.

\begin{table*}[t]
\centering
\caption{Total variation (TV) distance and KL divergence (KL) of posteriors inferred from linear Gaussian ANM datasets ($n=500$). Achieving both lower TV and KL values is better.}

\scalebox{1.1}{
  \begin{sc}  
\begin{tabular}{lcc cc cc}
\toprule
\multirow{2}{*}{} &
\multicolumn{2}{c}{$d=2$} &
\multicolumn{2}{c}{$d=3$} &
\multicolumn{2}{c}{$d=4$} \\
\cmidrule(lr){2-3}\cmidrule(lr){4-5}\cmidrule(lr){6-7}
& TV & KL & TV & KL & TV & KL \\
\midrule
MC3      & $\mathbf{0.01 \pm 0.00}$ & $\mathbf{0.01 \pm 0.00}$ & $\mathbf{0.10 \pm 0.01}$ & $\mathbf{0.09 \pm 0.01}$ & $0.42 \pm 0.10$ & $6.52 \pm 0.91$ \\
DDS      & $0.31 \pm 0.10$ & $1.94 \pm 0.28$ & $0.32 \pm 0.08$ & $1.81 \pm 0.21$ & $0.45 \pm 0.11$ & $5.52 \pm 0.73$ \\
Proposed & $0.26 \pm 0.06$ & $1.01 \pm 0.46$ & $0.21 \pm 0.09$ & $0.34 \pm 0.09$ & $\mathbf{0.34 \pm 0.08}$ & $\mathbf{1.51 \pm 0.28}$ \\
\bottomrule
\end{tabular}
\end{sc}
}
\label{table-posterior}
\end{table*}

\textbf{Evaluation Metrics:}
We use the structural Hamming distance (SHD) and F1 score (\Cref{subsec-metrics}). 
We report the mean and standard deviation (SD) 
over 20 randomly generated datasets.

\textbf{Performance under Sparse Graphs:} 
\Cref{fig-diffHNMNN} presents the SHD rate (i.e., the SHD divided by its maximum possible value, $\mathrm{SHD}/\binom{d}{2}$) and the F1 score.
Our method outperforms all baselines on both synthetic and semi-synthetic datasets, thereby demonstrating its effectiveness for moment-driven causal discovery.

On synthetic datasets (top row of \Cref{fig-diffHNMNN}), \textbf{MC3} and \textbf{DDS} yield poor F1 scores because their ANM-based likelihoods are misspecified under heteroscedasticity. \textbf{ICDH} and \textbf{HOST} perform worse at inferring the variance graph, suggesting that their HNM-based inference 
is driven by the underlying mean-graph structure and therefore overlooks the drivers of heteroscedasticity. These findings highlight the importance of moment-driven causal discovery approaches for understanding complex real-world phenomena.

On semi-synthetic datasets (bottom row of \Cref{fig-diffHNMNN}), which reflect the complex nonlinearity and heteroscedasticity of gene expression data, our method achieves substantial gains over the baselines—especially in the F1 score—underscoring its practical reliability and versatility. Although our method also performs best at $\nf = 50$, accurate inference in such high-dimensional settings remains challenging due to the accumulation of heteroscedastic noise. Incorporating more advanced heteroscedastic noise regression techniques to tackle this challenge is left as our future work.

\textbf{Performance under Dense Graphs:} We also confirm the superior performance of our method under dense mean and variance graphs. 
\Cref{fig-dense} presents the results on synthetic datasets with dense graphs (ER model with 4 expected edges per node). Despite the increased complexity of the underlying graphs, our method consistently outperforms all baselines in both SHD rate and F1 score, demonstrating its performance robustness.

\subsubsection{Moment-Agnostic Graph Inference} \label{subsubsec-synth-mo}
To ensure a fair comparison with baselines,
we test the moment-agnostic causal graph inference performance of our method.

\begin{figure}[t]
  \includegraphics[height=6cm]{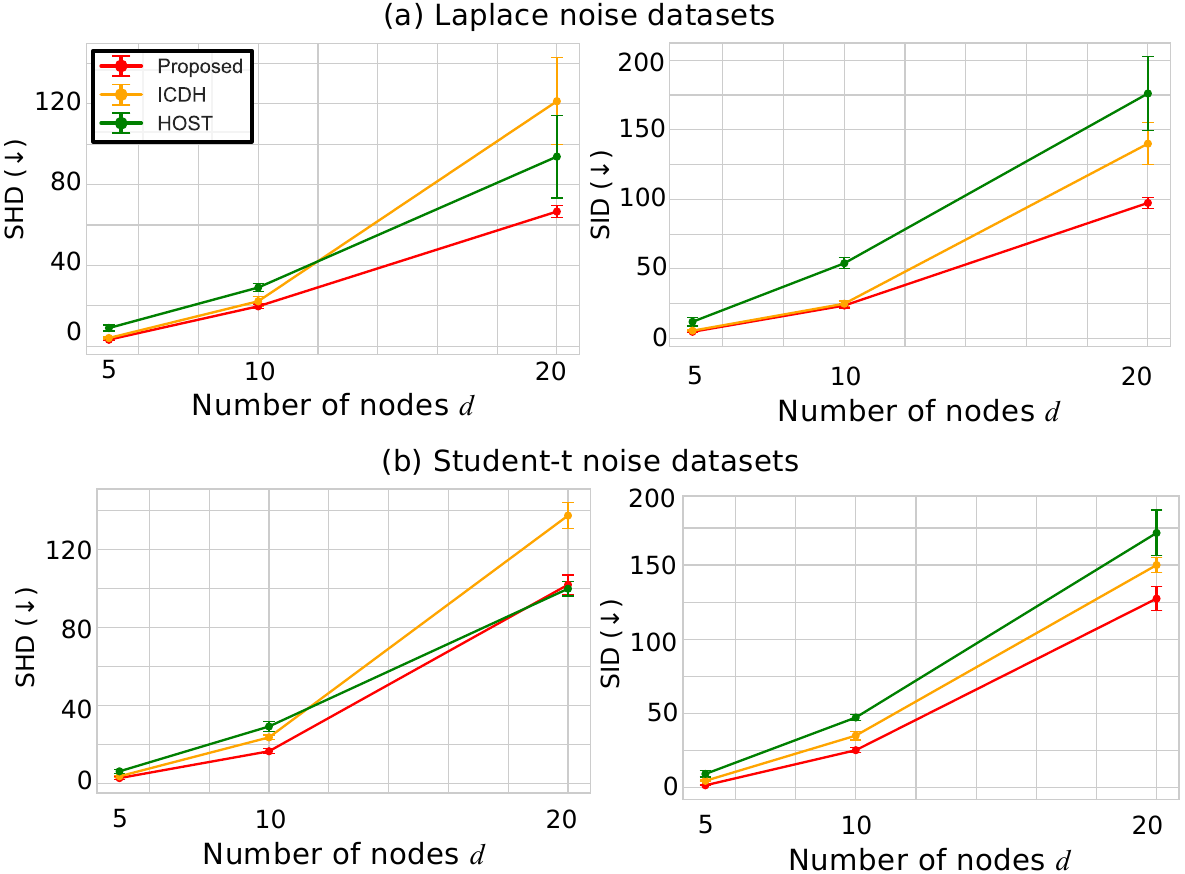}
  \centering 
  \caption{Moment-agnostic causal graph inference performance on non-Gaussian synthetic datasets ($n=500$) with (a) Laplace noise and (b) student-t noise. 
  Achieving \textbf{both} lower SHD (left) and lower SID (right) is better.} 
  \label{fig-nonGaussian}
\end{figure}

\textbf{Data:} We use synthetic datasets sampled from linear Gaussian ANMs and nonlinear Gaussian HNMs.
For linear Gaussian ANMs, we randomly draw both linear coefficients and the causal graph structures from uniform distributions.
Regarding nonlinear HNMs, we use the same data generation process as in \Cref{subsubsec-synth-mv} by sampling sparse moment-agnostic causal graphs from ER models.

\textbf{Posterior Approximation Quality:} We first assess the quality of the posterior approximation, using linear Gaussian ANMs, where the ground-truth posterior can be computed exactly for small $\nf$ \citep{annadani2023bayesdag}.

\Cref{table-posterior} presents the total variation (TV) distance and KL divergence (KL) between the inferred and true posteriors over moment-agnostic causal graphs. 
Our method consistently outperforms \textbf{DDS}, despite the increased model complexity for inferring two separate graphs. 
Since \textbf{MC3} is specifically designed for linear Gaussian ANMs,
 it achieves the best posterior approximation performance when $\nf=2, 3$. However, due to its poor scalability, 
 it becomes worse when $\nf = 4$. 
These results demonstrate the effectiveness of our variational inference framework for posterior approximation. 

\textbf{Robustness to Non-Gaussian Noise:} We next evaluate the performance on non-Gaussian HNM datasets, using SHD and structural intervention distance (SID) \citep{peters2015structural}. Due to the absence of baselines for multi-variate non-Gaussian HNMs, we consider \textbf{ICDH} and \textbf{HOST}, both of which are designed for Gaussian HNMs.\footnote{To the best of our knowledge, the official implementation of SkewScore \citep{li2024local}---the only method designed for multivariate non-Gaussian HNMs---is not publicly available.} 

\Cref{fig-nonGaussian} presents the results on the Laplace and student-t datasets.
Despite the likelihood model misspecification, 
our method consistently outperforms the baselines in terms of both SHD and SID scores, 
suggesting strong empirical robustness to noise misspecification.
A possible reason is that our method may mitigate the impact of noise distribution misspecification by downweighting the gradients from extreme observed noises, owing to our parameter learning strategy presented in \Cref{subsec-difficulty}.

\textbf{Additional Experiments for Investigating Model Robustness:} Finally, we test the performance when the data are generated from different identifiable SCM classes, namely, nonlinear Gaussian ANMs and HNMs. On these homoscedastic and heteroscedastic datasets, our method again achieves comparable performance to the ANM-based and HNM-based baselines, demonstrating its robustness in moment-agnostic graph inference. See \Cref{subsec-ANM-HNM} for the detailed results and experimental settings.

\begin{table*}[t]
  \centering
  \caption{Moment-agnostic graph inference performance on Sachs dataset. $\downarrow$ and $\uparrow$ denote "lower is better" and "higher is better".}
  \scalebox{1.1}{
  \begin{sc}  
  \begin{tabular}{lcccccc}
    \toprule
    & \multicolumn{2}{c}{$n=100$} & \multicolumn{2}{c}{$n=200$} & \multicolumn{2}{c}{$n=853$}  \\
    & SHD ($\downarrow$) & F1 ($\uparrow$) & SHD ($\downarrow$) & F1 ($\uparrow$) & SHD ($\downarrow$) & F1 ($\uparrow$) \\
           \midrule
  MC3      & 22.6 $\pm$ 1.1     & 0.19 $\pm$ 0.04  & 20.3 $\pm$ 1.3     & 0.20 $\pm$ 0.07  & 18.9 $\pm$ 0.7     & 0.14 $\pm$ 0.05  \\
  DDS     & 17.6 $\pm$ 0.7     & 0.20 $\pm$ 0.03   & 16.6 $\pm$ 0.6     & 0.20 $\pm$ 0.02   & 15.0 $\pm$ 0.7     & 0.28 $\pm$ 0.02   \\
  ICDH     & 20.0 $\pm$ 1.0 & \textbf{0.21 $\pm$ 0.01}  & 17.5 $\pm$ 0.5 & 0.22 $\pm$ 0.01  & 14.5 $\pm$ 0.5 & 0.27 $\pm$ 0.03  \\
  HOST     & 16.1 $\pm$ 0.5     & 0.19 $\pm$ 0.01   & 15.0 $\pm$ 0.3     & \textbf{0.33 $\pm$ 0.02}   & \textbf{13.5 $\pm$ 0.8}     & \textbf{0.38 $\pm$ 0.03}   \\
  Proposed & \textbf{16.0 $\pm$ 0.3}     & 0.20 $\pm$ 0.02   & \textbf{14.9 $\pm$ 0.4}     & 0.31 $\pm$ 0.02   & 13.7 $\pm$ 0.5     & 0.36 $\pm$ 0.02   \\
  \midrule
  Proposed $+ 25\%$ & 14.9 $\pm$ 0.4     & 0.34 $\pm$ 0.03   & 14.8 $\pm$ 0.5     & 0.35 $\pm$ 0.02   & 13.4 $\pm$ 0.4     & 0.37 $\pm$ 0.02   \\
  Proposed $+ 50\%$ & \textbf{13.2 $\pm$ 0.5}     & \textbf{0.36 $\pm$ 0.02}   & \textbf{13.2 $\pm$ 0.3}     & \textbf{0.36 $\pm$ 0.02}   & \textbf{13.1 $\pm$ 0.2}     & \textbf{0.45 $\pm$ 0.03}   \\
  \bottomrule 
  \end{tabular}
\end{sc}
  }
  \label{table-sachs}
  \end{table*}

\begin{table}[t]
  \centering
  \caption{Run time on Sachs dataset ($n= 853$). Our \textsc{Proposed} needs about twice the run time of \textsc{DDS} due to two-graph inference and can use prior knowledge with small overhead.}
  \scalebox{1.1}{  
  \begin{sc}  
  \begin{tabular}{lc}
    \toprule
    & Run time [sec] \\
           \midrule
  DDS     & 367 $\pm$ 11   \\
  Proposed & 733 $\pm$ 19     \\
  \midrule
  Proposed $+ 25\%$ & 852 $\pm$ 38     \\
  Proposed $+ 50\%$ & 1260 $\pm$ 51     \\
  \bottomrule 
  \end{tabular}
\end{sc}
  }
  \label{table-runtime}
  \end{table}

\begin{table}[t]
  \centering
  \caption{Estimated posterior probability of MEK $\rightarrow$ ERK in inferred variance causal graph structure}
  \scalebox{1.05}{  
  \begin{sc}    
\begin{tabular}{llll}
\toprule
 & $n = 100$ & $n = 200$ &  $n = 853$ \\
\midrule
Proposed      & $0.585 \pm 0.009$ & $0.592 \pm 0.008$ & $0.620 \pm 0.019$ \\
\bottomrule
\end{tabular}
\end{sc}
  }
  \label{table-real-case-study}  
\end{table}

\subsection{Real-World Data Experiments} \label{subsec-real}

\textbf{Data:} We use a well-established benchmark dataset,
called the Sachs dataset \citep{sachs2005causal},
which contains $853$ observations of the expression levels of $\nf=11$ proteins in human cells. 
We test each method under $n=100$, $200$, and $853$ observations with $20$ random restarts.

\subsubsection{Moment-Agnostic Graph Inference}
Because reliable ground truth for the mean and variance graph structures is \textbf{unavailable}, for a rigorous comparison we evaluate graph-level inference performance only on moment-agnostic graphs.

\Cref{table-sachs} presents the results.
Although our method learns additional parameters for inferring two separate  graphs,
it achieves comparable performance to the state-of-the-art \textbf{HOST} method.

When it leverages randomly selected $25\%$ and $50\%$ of pairwise node orderings,  
the performance improves with an increase in their number, especially for small sample-size setups ($n=100$ and $200$), underscoring its practical utility under data scarcity.

Notably, our run time comparison results presented in \Cref{table-runtime} show that 
such performance gain of our prior knowledge incorporation approach comes with only a modest increase in run time, 
demonstrating its computational efficiency.
This computational overhead could be further reduced by employing more efficient QP solvers for solving the constrained least squares in Eq.  \eqref{eq-priork}.

\subsubsection{Case Study: Variance Edge Detection} \label{subsubsec-casestudy}
Using the Sachs dataset, we conduct a case study to assess our method's ability to identify the origin of heteroscedasticity in protein signaling network. This heteroscedasticity is well documented in the literature \citep{jacob2002convergence,filippi2016robustness,davies2020systems} and is statistically significant in the Sachs dataset ($p < 0.0096$ when performing the White test with random forest regression).

Guided by prior biological evidence \citep{filippi2016robustness}, we hypothesize that the MEK $\rightarrow$ ERK edge in the ground-truth protein regulatory network represents a causal relationship at the variance level, and we assess whether our method can identify this edge.

\Cref{table-real-case-study} presents the mean and standard deviation of the estimated posterior probability of the MEK $\rightarrow$ ERK edge across 20 runs (with different random seeds). Consistent with biological evidence, the estimated posterior probabilities remain consistently high even in low-sample regimes ($n=100$ and $200$), implying that our Bayesian framework successfully predicts the edge presence with high certainty by detecting the underlying heteroscedasticity.
These results 
underscore the effectiveness of our Bayesian, moment-driven approach 
in discovering the drivers of heteroscedasticity
under data scarcity settings in practical real-world applications.

\section{Related Work} \label{sec-related}

\subsection{HNM-based Causal Discovery}

HNMs are of growing interest as a flexible model that addresses heteroscedasticity in complex real-world data. 
Early work has shown bivariate identifiability 
\citep{khemakhem2021causal,xu2022inferring},
which is extended to multivariate settings \citep{strobl2023identifying,yin2024effective}.
Building on their identifiability results,
likelihood-based methods \citep{immer2023identifiability,kikuchi2022differentiable,tran2024robust,yin2024effective} and independence-test-based methods \citep{duong2023heteroscedastic,immer2023identifiability,lin2024skewness} have been proposed
for the point estimation of the causal graph.
Our work differs in two key aspects.
(\myra) Our inference targets are mean and variance causal graphs,
whose definition and identifiability are presented in \Cref{sec-problem}.
(\myrb) Unlike existing point estimation methods, 
we develop a Bayesian approach that infers 
the posterior distribution over causal graphs (\Cref{sec-method}), 
 which enables Monte Carlo approximation of posterior probabilities for structural features, such as the presence of an edge, a path, or a subgraph \citep{friedman2003being}.
 As demonstrated in our case study (\Cref{subsubsec-casestudy}),
 such uncertainty quantification is crucial for knowledge discovery.

\subsection{Bayesian Causal Discovery}

Bayesian causal discovery provides a principled way to quantify
the inference uncertainty of causal graph structures from finite data.
Existing approaches have addressed two inference targets
\citep{lorch2021dibs}.
The first is the posterior distribution $\pr(\biA \mid \mathcal{D})$
over a causal graph $\biA$:
\begin{align*}
  \pr(\biA \mid \mathcal{D}) \propto \pr(\biA) 
    \int \pr(\Theta \mid \biA) \pr(\mathcal{D} \mid \Theta, \biA) \ \mathrm{d}\Theta,
\end{align*}
where the likelihood parameters $\Theta$ are integrated out.
The second is the joint posterior 
$\pr(\biA,\Theta \mid \mathcal{D})$ over the graph and
the parameters.

The former target has been widely studied for quantifying graph-level
uncertainty
\citep{charpentier2022differentiable,cundy2021bcd,deleu2022bayesian,geiger1994learning,giudici2003improving,goudie2016gibbs,hoyer2009bayesian,kim2003inferring,madigan1995bayesian,ong2025metacadi},
whereas recent methods have focused on the latter joint posterior to quantify
uncertainty in both causal graphs and SCM parameters
\citep{annadani2023bayesdag,deleu2023joint,lorch2021dibs,toth2024effectivebayes}.
Such joint posterior inference is useful for downstream tasks
\citep{emezue2023benchmarking},
such as causal effect estimation
\citep{shalit2017estimating,johansson2022generalization,chikahara2024differentiable,iwata2024metalearning},
subgroup analysis \citep{bargagli2020causal,chikahara2022feature,zhao2022selective},
and causality-based algorithmic fairness
\citep{chikahara2021learning,chikahara2026fairness,kusner2017counterfactual,wu2019pc}.

Our framework addresses the former inference target by building on DDS
\citep{charpentier2022differentiable}, but extends its target
to mean and variance causal
graphs, i.e., $\pr(\biA^M,\biA^V \mid \mathcal{D})$.
We adopt this target because full joint posterior inference over
$\pr(\biA^M,\biA^V,\Theta \mid \mathcal{D})$ would require capturing
uncertainty in mean and variance function parameters for each variable,
which is substantially challenging due to the difficulty of fitting
heteroscedastic likelihoods discussed in \Cref{subsec-difficulty}.
That said, extending our method to such joint posterior inference is left as a promising direction for future work, 
which would enable a broader range of downstream applications, such as the Bayesian approach to the targeted intervention strategy illustrated in the drug discovery scenario in \Cref{example-drug} \citep{agrawal2019abcd,tigas2022interventions}.

\section{Conclusion} \label{sec-conc}

This paper is the first to propose moment-driven causal discovery, 
advancing causal discovery by enabling moment-specific reasoning about complex causal mechanisms.
As a first step,
we derive identifiability conditions for mean and variance causal graphs and
propose a theoretically grounded variational inference framework that learns them directly,
without relying on intermediate moment-agnostic graph estimates.
This direct approach yields principled uncertainty quantification for the inferred graphs
and can further improve sample efficiency by incorporating prior knowledge.

As a promising direction for future work,
we aim to establish a general identifiability theory for moment-dependent causal graph structures.
This challenging yet important problem will enable the inference of \textit{higher-order moment causal graphs}, which would elucidate how distributional characteristics like skewness and kurtosis are determined in complex real-world phenomena.

\begin{acks}
This work was supported by JST ACT-X (JPMJAX23CF).
\end{acks}

\bibliographystyle{ACM-Reference-Format}
\bibliography{main}

\appendix

\section{Assumptions} \label{sec-asmp}

This section describes the assumptions for deriving the identifiability conditions in \Cref{sec-proof}.

We make a standard assumption in causal discovery:
\asmpCsuff*
Assumption \ref{asmp-csuff} requires that there is no unobserved common cause of observed variables $\biX$; this assumption is widely used in the causal inference literature \citep{glymour2019review}.
Satisfying Assumption \ref{asmp-csuff} and the DAG condition on $\cg$ 
ensures
the causal Markov condition \citep{pearl2009causality},
which assumes that each variable is conditionally independent of 
the variables represented as its non-descendants in $\cg$. 

In addition, we make another classical assumption 
that ensures identifiability:
\asmpCm*
Causal minimality requires that 
joint distribution $\pr(\biX)$ be Markov with respect to (moment-agnostic) causal graph $\cg$ but not with respect to any proper subgraph of $\cg$ \citep{spirtes2001causation}.
Roughly speaking, satisfying it with respect to causal graph $\cg$ implies that
 removing any edge from $\cg$ will violate the causal Markov condition, and
hence the conditional independence relations in joint distribution $\pr(\biX)$ 
will no longer be compatible with $\cg$.
Causal minimality, 
which is standard for deriving the identifiability under a restricted class of SCMs, is weaker than the faithfulness assumption \citep{peters2014causal}.

Furthermore, we make two additional assumptions that are specific to our results:
\asmpDag*
\asmpOrder*
As mentioned in \Cref{subsec-identi}, 
both assumptions are needed to ensure that  moment-agnostic causal graph $\cg$ is a DAG,
which has the same permutation as $\cgm$ and $\cgv$:
\begin{corollary}
  \label{cor-DAG}
  Under Assumptions \ref{asmp-dag} and \ref{asmp-order}, 
  moment-agnostic causal graph $\cg$ is a DAG with the same permutation as the mean and variance causal graphs $\cgm$ and $\cgv$.
\end{corollary}
\begin{proof}
  Let $\pi$ be a permutation (i.e., a valid topological ordering) of both $G^{M}$ and $G^{V}$. We first show that $\pi$ is also a valid topological ordering for their union graph $G$ and then prove that $G$ is a DAG.

  Since a directed path in a DAG is a concatenation of a finite number of directed edges, it is sufficient to verify that $\pi(u) < \pi(v)$ holds for every edge $(u, v) \in G$. By assumption, $\pi$ is a valid topological ordering over both $G^{M}$ and $G^{V}$.
  Thus, for every $(u, v) \in G$, regardless of whether it is present in $G^{M}$, $G^{V}$, or both, we have $\pi(u) < \pi(v)$, and hence $\pi$ is a valid topological ordering for $G$.
  
  Next, we show that the union graph $G$ is acyclic. Suppose that $G$ contains a directed cycle $v_1 \rightarrow \cdots \rightarrow v_k \rightarrow v_1$ over the node(s) $v_1, \ldots, v_k$ ($k \geq 1$). Then, since $\pi$ is a topological ordering of $G$, it must satisfy $\pi(v_1) < \cdots < \pi(v_k) < \pi(v_1)$;
  however, $\pi(v_1) < \pi(v_1)$ is a contradiction.
  Therefore, such a cycle cannot exist.
  Hence, $G$ is a DAG and $\pi$ is its valid topological ordering.
  \end{proof}

\begin{remark} \label{remark-asmp}
  The DAG condition on the causal graph (Corollary \ref{cor-DAG}) is a standard assumption for showing the identifiability in the HNM literature \citep{xu2022inferring,strobl2023identifying,lin2024skewness,yin2024effective}.
  Accordingly, the corresponding Assumptions \ref{asmp-dag} and \ref{asmp-order} in our setup are also standard. 
\end{remark}

\section{Identifiability Proof} \label{sec-proof}

This section presents the proof of \Cref{th2}.
We take three steps: 
\begin{enumerate}[leftmargin=1cm]
  \renewcommand{\labelenumi}{(\roman{enumi})}
  \item We show the identifiability of moment-agnostic causal graph $\cg$ of our mean-variance HNM for bivariate case $\nf = 2$ (\Cref{subsubsec-bivariate}).
  \item We extend these results to a multivariate setup $\nf > 2$ (\Cref{subsubsec-multivariate}).
  \item We prove \Cref{th2} by deriving the sufficient conditions for the identification of mean and variance causal graphs $\cgm$ and $\cgv$ (\Cref{subsec-mean-var-graph}).
\end{enumerate}


\subsection{Identifiability of Moment-Agnostic Causal Graph of Mean-Variance HNM} \label{subsec-moment-agnostic}

\subsubsection{Bivariate Case ($\nf = 2$)} \label{subsubsec-bivariate}

We prove bivariate identifiability of a moment-agnostic causal graph $\cg$ in two steps:
\begin{enumerate}[leftmargin=1cm]
  \renewcommand{\labelenumi}{(\myra-\arabic{enumi})}
  \item We illustrate \textit{unidentifiable scenarios}, 
where the causal direction between $X_1$ and $X_2$ in moment-agnostic causal graph $\cg$ is unidentifiable from joint distribution $\pr(X_1, X_2)$.
To derive such scenarios, we apply the results of \citet{khemakhem2021causal},
which also hold for our mean-variance HNM under our assumptions. 
  \item We derive the sufficient conditions that exclude these unidentifiable scenarios.
\end{enumerate}
To highlight the differences between the original HNM and our mean-variance HNM, below we 
overview the existing results on the original HNM and then extend them to our mean-variance HNM.


\textbf{Results on the original HNM:}
\citet{khemakhem2021causal} show that
two (original) HNMs with opposite causal directions can induce the same distribution $\pr(X_1, X_2)$ under the following two scenarios:
\begin{theorem}[{\citet[Theorem 2]{khemakhem2021causal}}]\label{th-khemakhem}
  \label{th:iden_gen}
  Let $E_1, E_2 \sim \mathcal{N}(0, 1)$ be standard Gaussian noises,
  and let 
  $m_1, m_2, v_1$ and $v_2$ be the twice-differentiable scalar functions on $\R$ 
  that satisfy $v_1(\cdot) > 0$ and $v_2(\cdot) > 0$.
  Assume that $X_2$'s values are given by the forward model:  
  \begin{equation} \label{modmodelxy}
  X_2 = m_2(X_1) + v_2(X_1) E_2,
  \end{equation}
  where $E_2 \indep X_1$. Suppose that joint distribution $\pr(X_1, X_2)$ is also compatible with the backward model:
  \begin{equation} \label{modmodelyx}
    X_1 = m_1(X_2) + v_1(X_2) E_1,
  \end{equation}
  where $E_1 \indep X_2$. Then either of the two scenarios must hold:
  \begin{enumerate}[leftmargin=1cm]
      \item \label{scenario1} $(m_2, v_2) = \left(\frac{B}{H}, \frac{1}{H}\right)$ and $(m_1, v_1) = \left(\frac{B'}{H'}, \frac{1}{H'}\right)$ where $B$ and $B'$ are polynomials of degree two or less, $H > 0$ and $H' > 0$ are two-order polynomials, and $\pr(X_1)$, $\pr(X_2)$ are strictly log-mix-rational-log.\footnote{See \citet[Definition 1]{khemakhem2021causal} for the definition of the (strictly) log-mix-rational-log density.}
      \item \label{scenario2} $m_1, m_2$ are linear, $v_1, v_2$ are constant, and $\pr(X_1)$, $\pr(X_2)$ are Gaussian distribution.
  \end{enumerate}
  \end{theorem}
  Assuming that $E_1$ and $E_2$ are (standard) Gaussian noises,
  Theorem \ref{th-khemakhem} shows that the moment-agnostic graph becomes unidentifiable only in Scenarios (\ref{scenario1}) and (\ref{scenario2}). 
  While $v_{\nfi}$ ($\nfi =1, 2$) is a constant under Scenario (\ref{scenario2}),
  it \textbf{cannot} be a constant under Scenario (\ref{scenario1})
  because it must be the inverse of \textbf{two-order} polynomial function.
  
  Building on this result, 
  \citet{yin2024effective} provide the identifiability conditions for the original HNM:
  \begin{theorem}[{\citet[Theorem 2]{yin2024effective}}]\label{th-yin-bi}
    Under Assumptions \ref{asmp-csuff}, \ref{asmp-cm}
    and the condition that the (moment-agnostic) causal graph $\cg$ is a DAG,
    $\cg$ is identifiable
    from observational data distribution $\pr(X_1, X_2)$
    if $\pr(X_1, X_2)$ is generated from the (original) HNM in \eqref{eq-HNM} that satisfies the following conditions for each $\nfi = 1, 2$:
    (a) $m_{\nfi}$ is a nonlinear function, 
    (b) $v_{\nfi}$ is a piecewise function, and 
    (c) $E_{\nfi}$ is Gaussian noise.
  \end{theorem}
  Condition (c) is necessary to ensure the Gaussian noise assumption in \Cref{th-khemakhem}.
  Conditions (a) and (b) suffice to rule out
  Scenarios \ref{scenario2} and \ref{scenario1}, respectively:
  Nonlinear $m_{\nfi}$ excludes the linearity required in Scenario~\ref{scenario2},
  and piecewise $v_{\nfi}$ cannot be expressed as a polynomial form in Scenario \ref{scenario1}.

\textbf{Extension to mean-variance HNM:}
We now extend \Cref{th-yin-bi} to our mean-variance HNM 
by taking into account the difference between the original and our mean-variance HNMs.

As discussed in \Cref{subsec-problem}, 
any instance of the mean-variance HNM can be reformulated as the original HNM 
by incorporating input masking into the functions in the original HNM.
To avoid confusion,
we use the notation 
$m^{\mathrm{O}}_{\nfi}(\pafi)$ and $v^{\mathrm{O}}_{\nfi}(\pafi)$ ($\nfi = 1, 2$) 
for these functions in this section. 
Using this notation, 
the structural equation of $X_2$ when $X_1 \rightarrow X_2$ can be formulated as
\begin{align*}
  X_2 = m_2(B^M_{\mathrm{pa}^M(2)} X_1) + v_2(B^V_{\mathrm{pa}^V(2)} X_1) E_2 = m^{\mathrm{O}}_2(X_1) + v^{\mathrm{O}}_2(X_1) E_2,
\end{align*}
where $B^M_{\mathrm{pa}^M(2)}, B^V_{\mathrm{pa}^V(2)} \in \{0, 1\}$ are binary masking variables;
for instance, if $B^M_{\mathrm{pa}^M(2)} = 0$, 
then $m^{\mathrm{O}}_2(X_1)$ becomes constant with respect to $X_1$.
Under causal minimality (Assumption \ref{asmp-cm}),
$B^M_{\mathrm{pa}^M(2)}$ and $B^V_{\mathrm{pa}^V(2)}$ 
are assumed to be not simultaneously $0$,
ensuring that at least either of $m^{\mathrm{O}}_2(X_1)$ and $v^{\mathrm{O}}_2(X_1)$ depends on $X_1$.

Fortunately, 
\Cref{th-khemakhem} already covers the cases where the functions are constants, 
as it only assumes their twice-differentiability, which is trivially satisfied by constant functions.  
Therefore, 
the unidentifiable scenarios for our mean-variance HNM 
are again limited to Scenarios \ref{scenario1} and \ref{scenario2}.
We show that Conditions (a), (b), and (c) in \Cref{th-yin-bi} are sufficient 
to exclude these scenarios:
\begin{theorem}\label{th-bi}
  Under Assumptions \ref{asmp-csuff}, \ref{asmp-cm}, \ref{asmp-dag}, and \ref{asmp-order},
  moment-agnostic causal graph $\cg$ of the mean-variance HNM is identifiable
  from distribution $\pr(X_1, X_2)$
  if for $\nfi = 1, 2$,
  (a) $m_{\nfi}$ is a nonlinear function, 
  (b) $v_{\nfi}$ is a piecewise function, and 
  (c) $E_{\nfi}$ is Gaussian noise.
\end{theorem}
\begin{proof}
  The forward and backward models for the mean-variance HNMs are given by
  \begin{align}
    &X_2 = m_2(B^M_{\mathrm{pa}^M(2)} X_1) + v_2(B^V_{\mathrm{pa}^V(2)} X_1) E_2 = m^{\mathrm{O}}_2(X_1) + v^{\mathrm{O}}_2(X_1) E_2 , \label{eq-MVHNM-bi-xy} \\
    &X_1 = m_1(B^M_{\mathrm{pa}^M(1)} X_2) + v_1(B^V_{\mathrm{pa}^V(1)} X_2) E_1 = m^{\mathrm{O}}_1(X_2) + v^{\mathrm{O}}_1(X_2) E_1 , \label{eq-MVHNM-bi-yx}
  \end{align}
  where $m^{\mathrm{O}}_{\nfi}(X_{\nfi}) \coloneqq m_{\nfi}(B^M_{\mathrm{pa}^M(\nfi)} X_{\mathrm{pa}(\nfi)})$ and $v^{\mathrm{O}}_{\nfi}(x) \coloneqq v_{\nfi}(B^V_{\mathrm{pa}^V(\nfi)} X_{\mathrm{pa}(\nfi)})$ ($\nfi = 1, 2$) denote the functions in the original HNM, and $B^M_{\mathrm{pa}^M(\nfi)}, B^V_{\mathrm{pa}^V(\nfi)} \in \{0, 1\}$ are binary masking variables.

  Under Condition (c), unidentifiable scenarios are restricted to Scenarios \ref{scenario1} and \ref{scenario2} in \Cref{th-khemakhem}.
  As with the ANM cases \citep[Proposition 17]{peters2014causal},
  causal minimality (Assumption \ref{asmp-cm}) implies 
  the non-constancy of functions and requires that
  functions $m^{\mathrm{O}}_{\nfi}$ and $v^{\mathrm{O}}_{\nfi}$ in Eqs. \eqref{eq-MVHNM-bi-xy} and \eqref{eq-MVHNM-bi-yx} not simultaneously be constant with respect to any input variable in $\pafi$.
  Since this requirement does not allow $B^M_{\mathrm{pa}^M(\nfi)} = B^V_{\mathrm{pa}^V(\nfi)} = 0$,
  under Conditions (a) and (b) on the mean and variance functions $m_{\nfi}$ and $v_{\nfi}$,
  the functions in the original HNM belong to either of the following three cases:
  \begin{enumerate}[leftmargin=0.7cm]
    \item \label{case1} nonlinear $m^{\mathrm{O}}_{\nfi}$ and constant $v^{\mathrm{O}}_{\nfi}$ (i.e., $B^M_{\mathrm{pa}^M(\nfi)} = 1$ and $B^V_{\mathrm{pa}^V(\nfi)} = 0$),
    \item \label{case2} constant $m^{\mathrm{O}}_{\nfi}$ and piecewise (\textbf{but not constant}) $v^{\mathrm{O}}_{\nfi}$ (i.e., $B^M_{\mathrm{pa}^M(\nfi)} = 0$ and $B^V_{\mathrm{pa}^V(\nfi)} = 1$), 
    \item \label{case3} nonlinear $m^{\mathrm{O}}_{\nfi}$ and piecewise $v^{\mathrm{O}}_{\nfi}$ (i.e., $B^M_{\mathrm{pa}^M(\nfi)} = B^V_{\mathrm{pa}^V(\nfi)} = 1$). 
  \end{enumerate}
  In Case \ref{case1}, nonlinear $m^{\mathrm{O}}_{\nfi}$ and constant $v^{\mathrm{O}}_{\nfi}$ exclude Scenario \ref{scenario2} and \ref{scenario1}, respectively.  
  In Case \ref{case2}, piecewise but not constant $v^{\mathrm{O}}_{\nfi}$ suffices to exclude both scenarios.  
  In Case \ref{case3}, both scenarios are excluded in the same way as Case \ref{case1}. 
  Hence, all three cases exclude Scenarios \ref{scenario1} and \ref{scenario2}.
  This proves \Cref{th-bi}.
\end{proof}

\subsubsection{Multivariate Case ($\nf > 2$)} \label{subsubsec-multivariate}

We extend the bivariate identifiability result in \Cref{th-bi} to the multivariate case $\nf > 2$:
\begin{theorem}\label{th-multi}
  Under Assumptions \ref{asmp-csuff}, \ref{asmp-cm}, \ref{asmp-dag} and \ref{asmp-order},
  moment-agnostic causal graph $\cg$ of the mean-variance HNM is identifiable
  from distribution $\pr(\biX)$
  if for $\nfi = 1, \dots, \nf$,
  (a) $m_{\nfi}$ is a nonlinear function, 
  (b) $v_{\nfi}$ is a piecewise function, and 
  (c) $E_{\nfi}$ is Gaussian noise.
\end{theorem}
The proof proceeds in the same way as that of \citet[Theorem 2]{yin2024effective}, which we describe below.
\begin{proof}
  Suppose the joint distribution $\pr(\biX)$ is generated from a mean-variance HNM with moment-agnostic causal graph $\cg$.  
  Assume also that the same distribution can be obtained from another mean-variance HNM with a different moment-agnostic graph $\cg'$, i.e., $\cg \neq \cg'$.
  Assumptions \ref{asmp-dag} and \ref{asmp-order} imply that $\cg$ and $\cg'$ are both DAGs (Corollary \ref{cor-DAG}).
  Therefore, under Assumptions \ref{asmp-csuff}, and \ref{asmp-cm},
  $\pr(\biX)$ satisfies the causal Markov condition and causal minimality with respect to $\cg$ and $\cg'$.
  In such cases, according to Proposition 29 of \citet{peters2014causal},
  there is a pair of variables, $L, Y \in \biX$,
  such that for three variable subsets, $\biQ = \biX_{\mathrm{pa}^{\cg}(Y)} \backslash \{L\}$, $\biR = \biX_{\mathrm{pa}^{\cg'}(L)} \backslash \{Y\}$,
  and $\biS = \biQ \cup \biR$, the following conditions hold:
  \begin{align}
    &L \rightarrow Y \text{ in } \cg; \quad L \leftarrow Y \text{ in } \cg' \label{eq-proof1}\\
    &\biS \subseteq \biX_{\mathrm{nd}^{\cg}(Y)} \backslash \{L\}; \quad \biS \subseteq \biX_{\mathrm{nd}^{\cg'}(L)} \backslash \{Y\}, \label{eq-proof2}
  \end{align}
  where $\biX_{\mathrm{pa}^{\cg}(Y)}$ and $ \biX_{\mathrm{pa}^{\cg'}(L)}$ are the variables expressed as the parents of $Y$ in $\cg$ and $L$ in $\cg'$, respectively, and $\biX_{\mathrm{nd}^{\cg}(Y)}$ and $\biX_{\mathrm{nd}^{\cg'}(L)}$ are the variables represented as the non-descendants of $Y$ in $\cg$ and of $L$ in $\cg'$, respectively. 
  From Lemma 37 in \citet{peters2014causal}, 
  if the structural equation of endogenous variable $X_{\nfi} \in \biX$ ($\nfi \in \{1, \dots, \nf\}$) is given by the following general formulation with scalar exogenous variable $E_{\nfi}$
  \begin{align}
    X_{\nfi} = f_{\nfi}(\biX_{\mathrm{pa}^{\cg}(\nfi)}, E_{\nfi})\quad \mbox{for}\ \nfi = 1, \dots, \nf,  \label{eq-SCM-scalarE}
  \end{align}
  then independence relation $X_{\nfi} \indep \biK$ holds for variable subset $\biK \subseteq \biX_{\mathrm{nd}^{\cg}(\nfi)}$.
  Since the structural equation of the mean-variance HNM is a special case of \eqref{eq-SCM-scalarE},
  we can apply this result to derive the following independence relations: 
  \begin{align}
    E_{Y} \indep \{L, \biS\} \quad \text{ and } \quad E_{L} \indep \{Y, \biS\} \label{eq-proof3}. 
  \end{align}
  Consider the conditioning of random variables $\biS$ on $\biS = \bis$ with $\pr(\bis) > 0$. 
  Let $Y^*$ and $L^*$ be random variables obeying conditional distributions $\pr(Y \mid \biS = \bis)$ and $\pr(L \mid \biS = \bis)$, respectively.  
  According to Lemma 36 in \citet{peters2014causal}, if the independence relations in \eqref{eq-proof3} hold, 
  then the values of $Y^*$ and $L^*$ are determined by 
  \begin{align}
    &Y^* = f_Y(L^*, \biq, E_{Y}); \label{eq-proof4}\\
    &L^* = f_L(Y^*, \bir, E_{L}), \label{eq-proof5}
  \end{align}
  where $\biq, \bir \in \bis$ are the values of variables $\biQ$ and $\biR$.
  Eqs. \eqref{eq-proof4} and \eqref{eq-proof5} imply that 
  joint distribution $\pr(Y^*, L^*)$ is compatible with 
  both the forward and backward models.
  However, this contradicts the bivariate identifiability result in \Cref{th-bi},  
  which rules out such coexistence under the mean-variance HNM.
  Thus, $\cg'$ cannot differ from $\cg$, and $\cg$ is uniquely identifiable from $\pr(\biX)$.
\end{proof}
\begin{remark} \label{remark-multi}
  Conditions (a), (b), and (c) in \Cref{th-multi} are weaker than conditions (A), (B), and (C) in \Cref{th2}: Variance function $v_{\nfi}$ can be a constant under the former but cannot be under the latter. 
  Therefore, conditions (A), (B), and (C) in \Cref{th2} are also sufficient to ensure identifiability, 
  meaning that they are also the identifiability conditions for moment-agnostic causal graph $\cg$. 
\end{remark}

\subsection{Identifiability of Mean and Variance Causal Graphs of Mean-Variance HNM} \label{subsec-mean-var-graph}

Based on the identifiability of moment-agnostic causal graph $\cg$ in \Cref{th-multi},
we prove that the mean and variance causal graphs $\cgm$ and $\cgv$ are also identifiable under the stronger conditions.

In particular, we need additional conditions where 
\textbf{both} mean and variance functions in the mean-variance HNM
are not constant with respect to the inputs.
As noted in \Cref{remark}, excluding such constant cases is essential
to detect the presence of the inputs of the mean and variance functions. 
Since nonlinearity condition (a) in \Cref{th-multi} already requires that mean function $m_{\nfi}$ is not a constant function, we only have to add the non-constant condition on variance function $v_{\nfi}$:
\identifiability*
\begin{proof}
  From \Cref{th-multi} (and \Cref{remark-multi}), 
  Conditions (A), (B), and (C) are sufficient to identify 
  the moment-agnostic causal graph.
  Hence, it is sufficient to show the identifiability of mask matrices $\biB^{\mathcal{I}}_{\nfi} \in \{0, 1\}^{|\mathcal{I}| \times |\mathrm{pa}(\nfi)|}$ for separating a superset of parental variables $\pafi = \pamfi \cup \pavfi$ into mean and variance causes $\pamfi = \biB^{\mathrm{pa}^M(\nfi)}_{\nfi} \pafi$ and $\pavfi = \biB^{\mathrm{pa}^V(\nfi)}_{\nfi} \pafi$, for each $\nfi = 1, \dots, \nf$.

  As discussed in \Cref{subsec-problem},
  each endogenous variable $X_{\nfi}$ in the mean-variance HNM follows a conditional Gaussian: 
  \begin{align}
    X_{\nfi} \mid \pafi \sim \mathcal{N}(m_{\nfi}(\pamfi), c_{E_{\nfi}} \cdot (v_{\nfi}(\pavfi))^2 ), \label{eq-proof6}
  \end{align}
  where $c_{E_{\nfi}} \coloneqq \V[E_{\nfi}]$ is the variance of zero-mean noise $E_{\nfi}$. 
  Using mask matrices, the conditional Gaussian in \eqref{eq-proof6} can be equivalently rewritten as
  \begin{align}
    X_{\nfi} \mid \pafi \sim \mathcal{N}(m_{\nfi}(\biB^{\mathrm{pa}^M(\nfi)}_{\nfi} \pafi), c_{E_{\nfi}} \cdot (v_{\nfi}(\biB^{\mathrm{pa}^V(\nfi)}_{\nfi} \pafi))^2 ). \label{eq-proof6-original}
  \end{align}

  Suppose that under different mean and variance functions,  
  this conditional Gaussian distribution can be identical for different parental variables $\pamfi \neq \pammfi$ and $\pavfi \neq \pavvfi$:
  \begin{align}
    \begin{aligned}
    &\mathcal{N}(m_{\nfi}(\biB^{\pamfi}_{\nfi} \pafi), c_{E_{\nfi}} \cdot (v_{\nfi}(\biB^{\pavfi}_{\nfi} \pafi) )^2)  \\
    = 
    &\mathcal{N}(m'_{\nfi}(\biB^{\pammfi}_{\nfi} \pafi), c_{E_{\nfi}} \cdot (v'_{\nfi}(\biB^{\pavvfi}_{\nfi} \pafi) )^2 ), 
    \end{aligned} \label{eq-proof6-unidentifiable}
  \end{align}
  where $m_{\nfi} \neq m'_{\nfi}$, $v_{\nfi} \neq v'_{\nfi}$, $\biB^{\pamfi}_{\nfi} \neq \biB^{\pammfi}_{\nfi}$, and $\biB^{\pavfi}_{\nfi} \neq \biB^{\pavvfi}_{\nfi}$. Distributional equality \eqref{eq-proof6-unidentifiable} implies that the conditional mean and variance are identical. Formally, for any $\pafi$'s values and for some $m_{\nfi} \neq m'_{\nfi}$, $v_{\nfi} \neq v'_{\nfi}$, $\biB^{\pamfi}_{\nfi} \neq \biB^{\pammfi}_{\nfi}$, and $\biB^{\pavfi}_{\nfi} \neq \biB^{\pavvfi}_{\nfi}$, the following equalities must hold:
  \begin{align}
    &m_{\nfi}(\biB^{\pamfi}_{\nfi} \pafi) = m'_{\nfi}(\biB^{\pammfi}_{\nfi} \pafi) \label{eq-proof6-unidentifiable-mean} \\
    &c_{E_{\nfi}} \cdot (v_{\nfi}(\biB^{\pavfi}_{\nfi} \pafi))^2 = c_{E_{\nfi}} \cdot (v'_{\nfi}(\biB^{\pavvfi}_{\nfi} \pafi) )^2.  \label{eq-proof6-unidentifiable-var} 
  \end{align}
  Since $\biB^{\pamfi}_{\nfi} \neq \biB^{\pammfi}_{\nfi}$, there exists at least one variable that is included in one of $\pamfi = \biB^{\mathrm{pa}^M(\nfi)}_{\nfi} \pafi$ and $\pammfi = \biB^{\pammfi}_{\nfi} \pafi$, but not in the other. Assume, without loss of generality, that $\pammfi$ contains the variables that are missing in $\pamfi$. Those variables are not inputs of function $m_{\nfi}$, yet Eq. \eqref{eq-proof6-unidentifiable-mean} must hold for any $\pafi$'s values. Therefore, these additional variables must have \textbf{no influence on the output of function $m'_{\nfi}$}, implying that $m'_{\nfi}$ must be constant with respect to these variables. However, this violates the non-constant conditions on $m'_{\nfi}$.  
  A similar discussion holds for variance function $v_{\nfi}$ in \eqref{eq-proof6-unidentifiable-var}. Thus we prove the identifiability of mean and variance causal graphs $\cgm$ and $\cgv$.
\end{proof}

\section{Overview of Differentiable Sampling Models} \label{sec-dds}

\subsection{Gumbel Softmax Distribution} \label{subsec-gsm}

The Gumbel Softmax distribution \citep{jang2017categorical} provides 
a differentiable approximation for addressing the non-differentiability of the sampling operation of categorical variables. 
This approximation is needed to learn the parameters of neural-network-based discrete sampling models
by propagating gradients in standard backpropagation.

The core idea lies in the traditional statistical technique called \textit{Gumbel-Max trick} \citep{gumbel1954statistical}, 
which performs sampling from a $C$-class categorical distribution 
with log probabilities $[\phi_0, \dots, \phi_{C-1}]$ ($C \geq 2$)
by taking $\text{argmax}$ over the log probabilities perturbed by Gumbel noise $g_c \sim \text{Gumbel}(0)$:
\begin{equation}
    Z = \underset{c \in \{0, \dots, C-1\}}{\text{arg max}}\  \left\{ \phi_c + g_c \right\}, \label{eq-GumbelMax}
\end{equation}
where $c \in \{0, \dots, C-1\}$ denotes a class index. 

The  $\text{argmax}$ operator of Gumbel-Max trick in \eqref{eq-GumbelMax} is non-differentiable with respect to parameters $\phi_0, \dots, \phi_{C-1}$,
thus hindering gradient-based parameter learning.
To make this $\text{argmax}$ operator differentiable,
Gumbel-Softmax trick replaces it with 
a differentiable softmax function:
\begin{equation}
    \tilde{Z}_{c} = \frac{\mathrm{e}^{(\log \phi_c + g_c)/\tau}}{\sum_{c'=0}^{C-1} {\mathrm{e}^{(\log \phi_{c'} + g_{c'})/\tau}}} \in [0, 1],
\end{equation}
where $g_c \sim \text{Gumbel}(0)$ is a standard Gumbel noise, and
$\tau$ is a temperature parameter that controls the smoothness of the softmax function. As $\tau \to 0$, vector $[\tilde{Z}_0, \dots, \tilde{Z}_{C-1}]$ approaches a one-hot vector, mimicking a categorical sample.

In this paper, 
we employ the Gumbel-Softmax distribution model with the number of classes $C= 2$
as a sampling model for upper-triangular matrix elements 
$U^M_{i, j}, U^V_{i, j} \in \{0, 1\}$ for $i, j \in \{1, \dots, \nf\}$.

\subsection{Permutation Sampling with SoftSort Function} \label{subsec-dperm}

\citet{charpentier2022differentiable} sample the continuous relaxation of a $\nf \times \nf$ permutation matrix 
by dealing with the non-differentiability of the categorical sampling operation over $\{1, \dots, \nf\}$.

Their sampling procedure is founded on the Gumbel Top-$K$ trick \citep{vieira2014gumbel},
which samples $K$ categorical values over $\{1, \dots, \nf\}$ ($\nf \geq 2$) without replacement by iteratively selecting the top $K$ values of the logits perturbed by standard Gumbel noise: 
\begin{equation}
    I_1, \dots, I_K = \underset{c \in \{1, \dots, \nf\}}{\text{arg top-K}} \{\log \psi_c + g_c \}, \label{eq-GumbelTopK}
\end{equation}
where $I_1, \dots, I_K \in \{1, \dots, \nf\}$ are the sampled class indices,
$\text{arg top-K}$ denotes an operation that selects the $K$ largest values from the inputs,
$\psi_c$ is the probability of the categorical distribution for class $c \in \{1, \dots, \nf\}$,
and $g_c \sim \text{Gumbel}(0)$ is the standard Gumbel noise.
By setting $K=\nf$, Eq. \eqref{eq-GumbelTopK} can be used to obtain
a permutation that sorts vector $[\log \psi_1 + g_1, \dots, \log \psi_{\nf} + g_{\nf}]$ in  ascending order.
However, the $\text{arg top-K}$ operator 
is not differentiable, as with the $\text{argmax}$ operator in \eqref{eq-GumbelMax}.

For this reason, \citet{charpentier2022differentiable} 
replace the $\text{arg top-K}$ in \eqref{eq-GumbelTopK} 
with a SoftSort function \citep{prillo2020softsort},
which is differentiable with respect to its inputs.
This function is defined by applying the softmax function 
to a pairwise distance matrix as 
\begin{equation}
    \text{SoftSort}_{\tau}(\biv) = \text{softmax}\left( -\frac{d(\biv, \text{sort}(\biv))}{\tau} \right), \label{eq-softsort}
\end{equation}
where $d(\biv, \biv')$ is a differentiable semi-metric function (e.g., the L1 distance $d(x, y) = \|\biv- \biv'\|_1$) that returns a pairwise distance matrix,
$\text{sort}(\biv)$ denotes an operation that rearranges 
the elements of vector $\biv$ in ascending order,
and $\tau$ is a hyperparameter that controls the smoothness.
By applying a SoftSort function in \eqref{eq-softsort} to the perturbed log probabilities, \citet{charpentier2022differentiable}  sample a continuous relaxation of permutation matrix, i.e., $\gbiPi \in \R^{\nf \times \nf}$.

As described in \Cref{subsubsec-dm},
we directly employ the above sampling procedure
for permutation matrix $\biPi$.

\section{Experimental Settings} \label{sec-exp-setup}

\subsection{Baselines} \label{subsec-baselines}

In our experiments, we compare the performance of our method with the following baselines:
\begin{itemize}[leftmargin=0.5cm]
  \item Metropolis-coupled Markov chain Monte Carlo (\textbf{MC3}) \citep{madigan1995bayesian,giudici2003improving}, which performs a Metropolis-Hastings-based sampling whose target distribution is the posterior distribution over Bayesian network structures. We consider a tractable posterior formulation with the linear Gaussian models and employ the R-based implementation downloaded from \url{https://github.com/rjbgoudie/structmcmc} \citep{goudie2016gibbs}, which is published under GNU General Public License v3 (GPL-3).
  \item Differentiable DAG sampling (\textbf{DDS}) \citep{charpentier2022differentiable}, which approximately infers the posterior over a single causal DAG by solving a variational inference problem. Based on nonlinear ANM formulation, it evaluates the expected score of each DAG using the squared reconstruction loss. 
  \item Identifiable causal discovery under heteroscedastic data (\textbf{ICDH}) \citep{yin2024effective}, which learns the weighted adjacency matrix of a moment-agnostic causal graph based on a nonlinear HNM to output a single point estimate. We employ the original source codes on \url{https://github.com/naiyuyin/ICDH} and set the threshold for the inferred weighted adjacency matrix to $0.2$.  
  \item Heteroscedastic causal structure learning (\textbf{HOST}) \citep{duong2023heteroscedastic}, which offers a point estimate of a causal graph by performing conditional independence testing based on a nonlinear HNM. We use the implementation downloaded from \url{https://github.com/baosws/HOST} and set the significance level of the conditional independence test to $\alpha = 0.05$.
\end{itemize}

\subsection{Settings of Each Method} \label{subsec-setup}

\begin{itemize}[leftmargin=0.5cm]
  \item \textbf{SCM parameterization:} For our method, we parameterize the mean and variance functions $m_{\nfi}$ and $v_{\nfi}$ of our mean-variance HNM in \eqref{eq-MVHNM} using  2-layered MLPs. The nonlinear functions for \textbf{DDS}, \textbf{ICDH}, and \textbf{HOST} are parameterized using the MLPs by following their default settings.
  \item \textbf{Hyperparameter tuning:} For each method, we tune the hyperparameters by conducting a grid search based on objective function values. Regarding our method, we conduct a search over batch size $B \in \{32, 64, 128, 256\}$ and the number of hidden neurons $h \in \{8, 16, 32, 64\}$ in the 2-layered MLPs. To do so, we use the \texttt{optuna} package \citep{akiba2019optuna} in Python Package Index (PyPI) (licensed under GNU GPL-2) with the fixed random seed. 
  \item \textbf{Optimization algorithm:} For our method, we can use any gradient-based optimization algorithm. In our experiments, following the ICDH method \citep{yin2024effective}, we use the Limited-memory Broyden–Fletcher–Goldfarb–Shanno (LBFGS) algorithm, which is a popular optimization scheme in the family of quasi-Newton methods. We employ the LBFGS wrapper for Pytorch on \url{https://gist.github.com/arthurmensch/c55ac413868550f89225a0b9212aa4cd}, which is released under the MIT License and provided on an "AS IS" basis. To implement our prior knowledge incorporation approach presented in \Cref{subsec-priork}, we employ the \texttt{quadprog} package in Python Package Index (PyPI) (licensed under GNU GPL-2). 
\end{itemize}

\subsection{Simulated Data Generation Processes} \label{subsec-synthdata}

This section describes the generation processes of the synthetic and semi-synthetic data used in \Cref{subsec-synth}.

\subsubsection{Mean-Variance HNM Datasets} \label{subsubsec-mvhnm-data}

We randomly sample the ground truth mean and variance causal graphs $\cgm$ and $\cgv$.
For sparse graph setups, we use an ER model with $1$ expected edge per node for $\nf = 5$ and $2$ expected edges per node for $\nf = 10, 20$, and $50$.\footnote{For $\nf = 5$, we use an ER model with $1$ expected edge, not $2$ expected edges, because the ER model with $2$ expected edges always produces a complete graph for $\nf = 5$, thus losing randomness.}
For dense graph setups, we use an ER model with $4$ expected edges per node for $\nf = 10, 20$.
To guarantee the acyclicity of the sampled adjacency matrices, we obtain upper-triangular matrices $\biUm$ and $\biUv$ by masking the lower-triangular elements of the sampled adjacency matrices 
and randomly permuting the node indices.

Using the randomly sampled causal graphs, 
we sample the data from a mean-variance HNM parameterized with randomly initialized MLPs.
To sample the values of variable $X_{\nfi}$ ($\nfi \in \{1, \dots, \nf\}$),
we formulate the structural equation as
\begin{align}
  X_{\nfi} = m_{\nfi}\left(\pamfi\right) + \mathrm{e}^{v_{\nfi}\left(\pavfi\right)}\ E_{\nfi}, \label{eq-synthdataHNM}
\end{align}
where $m_{\nfi}(\cdot)$ and $v_{\nfi}(\cdot)$ are MLPs with randomly initialized parameters, and $E_{\nfi} \sim \mathcal{N}(0, 1)$ is noise that follows the standard Gaussian distribution. As with \citet{yin2024effective}, we use 1 hidden layer with 100 neurons and sigmoid activation for MLPs $m_{\nfi}(\cdot)$ and $v_{\nfi}(\cdot)$.

\subsubsection{Linear Gaussian ANM Datasets} \label{subsubsec-linanm-data}

Using the moment-agnostic causal graph randomly sampled from the ER model, 
we generate linear Gaussian ANM datasets.

For each $X_{\nfi}$ ($\nfi \in \{1, \dots, \nf\}$),
we sample its value according to the following linear structural equation:
\begin{align}
  X_{\nfi} = \biw^{\top}_{\nfi} \pafi + E_{\nfi}, \label{eq-linanm}
\end{align}
where $E_{\nfi}$ is the standard Gaussian noise, and $\biw$ is a linear coefficient vector randomly sampled from a uniform distribution:
\begin{align}
w_{i, \nfi} = S_{i,\nfi}\cdot \bigl(0.5 + 0.8\,U_{i, \nfi}\bigr) \quad \text{for } i \in \mathrm{pa}(\nfi),
\end{align}
where $S_{i, \nfi} \sim \mathrm{Unif}\{+1,-1\}$ and $U_{i, \nfi} \sim \mathrm{Unif}(0,1)$ are independent.

\subsubsection{Nonlinear Gaussian ANM and HNM Datasets} \label{subsubsec-anmhnm-data}

We create ANM and HNM datasets
by randomly generating a ground truth moment-agnostic causal graph from the ER model, where the expected number of edges is given in the same way as the sparse graph setups described in \Cref{subsubsec-mvhnm-data}.

We generate the ANM datasets
by considering the structural equation for each variable $X_{\nfi}$ ($\nfi \in \{1, \dots, \nf\}$):
\begin{align}
  X_{\nfi} = m_{\nfi}(\pafi) + E_{\nfi}, \label{eq-ANM}
\end{align}
where $m_{\nfi}(\cdot)$ is a nonlinear function parameterized with a randomly initialized MLP, and $E_{\nfi} \sim \mathcal{N}(0, \sigma_{\nfi}^2)$ is a Gaussian noise whose variance is sampled from uniform distribution by $\sigma_{\nfi}^2 \sim \mathrm{U}(0.5, 2)$.

Regarding the HNM datasets,
we parameterize the HNM structural equation:
\begin{align}
  X_{\nfi} = m_{\nfi}\left(\pafi\right) + \mathrm{e}^{v_{\nfi}\left(\pafi\right)}\ E_{\nfi}, \label{eq-HNM-data}
\end{align}
using randomly initialized MLPs as $m_{\nfi}(\cdot)$ and $v_{\nfi}(\cdot)$ and standard Gaussian noise $E_{\nfi} \sim \mathcal{N}(0, 1)$. We formulate each MLP in the same way as in \Cref{subsubsec-mvhnm-data}.

\subsubsection{Nonlinear Non-Gaussian HNM Datasets} \label{subsubsec-nongauss-data}

We generate synthetic datasets with non-Gaussian heteroscedastic noise by modifying the generation process for the (original) HNM datasets in \Cref{subsubsec-anmhnm-data}. 
In particular, we replace the Gaussian noise $E_{\nfi}$ in Eq. \eqref{eq-HNM-data} with the following non-Gaussian noises:
\begin{itemize}
  \item Laplace noise with location 0 and scale 1: $E_{\nfi} \sim \text{Laplace}(0, 1)$
  \item Student-t noise with degree of freedom 3: $E_{\nfi} \sim \text{Student-}t(3)$
\end{itemize}

\subsubsection{Semi-Synthetic Datasets (SERGIO Datasets)} \label{subsubsec-sergio-data}

Based on the well-known fact that real-world gene regulatory networks often exhibit scale-free properties \citep{albert2005scale},
we randomly draw the ground truth mean and variance causal graphs for semi-synthetic datasets, 
using the SF model with the expected degree $2$.

Using the random causal graphs described above, we generate the semi-synthetic datasets.
We download the SERGIO simulator from \url{https://github.com/PayamDiba/SERGIO} \citep{dibaeinia2020sergio} (licensed under GNU GPL-3) and modify it such that the ground truth mean and variance causal graphs are different. 
As described in \Cref{subsec-synth}, this simulator simulates the gene expression data by sampling from the steady state of a stochastic differential equation (called a chemical Langevin equation (CLE)) that represents the rate of the biochemical reactions in a gene regulatory network.
Since this differential equation consists of both mean and additive noise terms 
following zero-mean white noise Gaussian processes, and both are affected by the identical set of parental variables (see the formulation detail in \citep{dibaeinia2020sergio}),
it can be regarded as a special case of the original HNM. 
For this reason, to obtain the datasets based on the mean-variance HNM,
we use different parental variables between the mean and noise terms in this differential equation 
and perform sampling from its steady state.

\subsection{Real-World Data} \label{subsec-realdata}

We download the Sachs dataset from the link 
\url{https://ln5.sync.com/dl/b442986b0#5xpiy2n2-q9j87qze-kydrb7wn-xgqjiw2c}, 
which is shared by \citet{charpentier2022differentiable}. 

\subsection{Performance Metrics} \label{subsec-metrics}

We evaluate the performance of each method 
with the following widely used metrics:
\begin{itemize}[leftmargin=0.5cm]
  \item Structural Hamming distance (\textbf{SHD}), which measures the number of edge additions, removals, and reversals to turn the inferred causal graph into a ground truth causal graph. In \Cref{fig-diffHNMNN}, we report the \textbf{SHD error rate}, defined as each method’s SHD divided by the maximum possible SHD, i.e., $\mathrm{SHD}/\binom{d}{2}$, where $d$ is the number of nodes.
  \item Expected structural Hamming distance ($\E$-\textbf{SHD}), which is an expected value of SHD. Let $\hat{\pr}(\biA \mid \mathcal{D})$ be an inferred posterior over DAG adjacency matrix $\biA$. Then the expected SHD is estimated using $n_{\mathrm{MC}}$ Monte Carlo samples as
  \begin{align}
    \E\text{-\textbf{SHD}} \simeq \frac{1}{n_{\mathrm{MC}}} \sum_{i=1}^{n_{\mathrm{MC}}} \text{\textbf{SHD}}(\biA^{(i)}, \biA^{\mathrm{True}}), \label{eq-E-SHD}
  \end{align}
where $\biA^{\mathrm{True}}$ is the ground truth causal graph, and $\biA^{(1)}, \dots, \biA^{(n_{\mathrm{MC}})} \sim \hat{\pr}(\biA \mid \mathcal{D})$ are the Monte Carlo samples drawn from inferred posterior $\hat{\pr}(\biA \mid \mathcal{D})$. We set the number of Monte Carlo samples to $n_{\mathrm{MC}} = 2000$.
  \item F1 score, which compares the presence of each inferred edge with the ground truth edge set. 
  \item Structural Intervention Distance (\textbf{SID}) \citep{peters2015structural}, which evaluates how correctly the inferred graph offers adjustment sets, whose marginalization yields the correct interventional distributions. By definition, this score metric can be applied only for moment-agnostic inference performance evaluation, not for mean and variance causal graph inference performance evaluation.
  \item Total variation (\textbf{TV}) and Kullback-Leibler divergence (\textbf{KL}), which measure the distributional discrepancy between the inferred and ground truth posteriors. 
\end{itemize}
As described in \Cref{sec-exp},
we evaluate $\E$-\textbf{SHD} for the Bayesian methods (\textbf{Proposed}, \textbf{MC3}, and \textbf{DDS}) and compute \textbf{SHD} for the point estimation methods (\textbf{ICDH} and \textbf{HOST}).

\subsection{Computing Infrastructure} \label{sec-computing}
	
For our experiments, we used a 64-bit Ubuntu machine with 2.9GHz AMD EPYC 7513 32-core (x2) CPUs, NVIDIA RTX A6000 PCI-EX16 Gen4 (x8) GPUs, and 512-GB RAM.

\section{Additional Experimental Results} \label{sec-additional-exp}








\subsection{Moment-Agnostic Causal Graph Inference Performance on ANM and HNM Datasets} \label{subsec-ANM-HNM}







Using the ANM and HNM datasets described in \Cref{subsubsec-anmhnm-data}, we evaluate the performance of our method in inferring the moment-agnostic causal graph structure.

\Cref{fig-sameANMNN,fig-sameHNMNN} present the performance on the ANM and HNM datasets, respectively. 
Although our method is not designed for moment-agnostic causal graph inference, it achieves comparable performance to the baselines tailored for the ANMs and the HNMs. 
These results demonstrate that despite the model complexity for inferring the mean and variance causal graphs, our method effectively learns the model parameters
by leveraging the differentiable model formulation (\Cref{subsubsec-dm}) and the heteroscedastic noise regression techniques (\Cref{subsec-difficulty}),
thus successfully inferring the moment-agnostic causal graph structure.

\subsection{Comparison with VarSort} \label{subsec-varsort}

VarSort method \citep{reisach2021beware}, which simply determines the permutation of the causal DAG based on the marginal variance order, is a standard baseline for testing the methods based on ANMs.

However, we exclude Varsort from our main comparison in \Cref{subsec-synth} and \Cref{subsec-real}. 
The reason is identical to the one discussed in \citet{yin2024effective}:
The marginal variance order estimation is unstable due to heteroscedastic noise,
leading to large variance in estimated causal node orderings.

Below, we elaborate on this point from a theoretical and empirical perspective.

\subsubsection{Theoretical Comparison}

We illustrate the formulation difference of marginal variance between the ANM and our mean-variance HNM.
From the law of total variance, 
the marginal variance of a variable $X_{\nfi}$ following the mean-variance HNM is given by 
\begin{align*}
  \V[X_{\nfi}] &= \V[\E[X_{\nfi} \mid \pamfi]] + \E[\V[X_{\nfi} \mid \pavfi]] \\
  &= \V[m_{\nfi}(\pamfi)] + \E[\left( v_{\nfi}(\pavfi)\right)^2],
\end{align*}
where $m_{\nfi}(\cdot)$ and $v_{\nfi}(\cdot)$ are the mean and variance functions of the mean-variance HNM, respectively.

Hence, the ANM and mean-variance HNM datasets have the following marginal variance:
\begin{itemize} [leftmargin=0.5cm]
\item Under homoscedasticity in the ANM datasets,
the second term is constant, as the variance function $v_{\nfi}(\cdot)$ is constant in the ANM (\Cref{sec-background}). 
\item Under heteroscedasticity in the mean-variance HNM datasets, however, the second term is non-constant, and the order of marginal variances can vary substantially.
\end{itemize}
For this reason, we have the following difference. Under the ANM datasets, the estimated marginal variance order can often be aligned with the causal node ordering, which is why the VarSort method serves as an important baseline for such data. Under the mean-variance HNM datasets, however, the estimated marginal variance order is less likely to be aligned with the causal node ordering.

\subsubsection{Empirical Comparison}

We present the empirical comparison results between our method and VarSort,
 using mean-variance HNM datasets (\Cref{subsubsec-mvhnm-data}).

\Cref{table-varsort} presents the (expected) SHD over 20 runs.
As expected, the performance of VarSort is unstable, with a large estimation variance.
These results suggest that its behavior is not robust but may align with ground truth purely by chance, leading to similar average SHD scores.

By contrast, the HNM-based methods, including our method, consistently achieve smaller variances in SHD, 
demonstrating their robustness in causal graph inference under heteroscedastic noise.

\newpage

\begin{table*}[t]
  \centering
  \caption{SHD score comparison with VarSort on mean-variance HNM datasets ($n=500$) with number of nodes $d = 5, 10$}
  \scalebox{1.1}{  
  \begin{sc}  
\begin{tabular}{lllll}
\toprule
Method &  Mean ($d=5$) &  Variance ($d=5$) & Mean ($d=10$) & Variance ($d=10$) \\
\midrule
Proposed & $\bm{2.34 \pm 0.20}$ &  $\bm{3.33 \pm 0.30}$& $\bm{15.5 \pm 1.2}$ & $\bm{19.3 \pm 1.1}$\\
ICDH	&$5.05 \pm 0.42$ & $4.55 \pm 0.40$	& $20.1 \pm 1.1$ & $20.9 \pm  1.2$ \\
HOST &	$4.84 \pm  0.49$ & $4.60 \pm 0.45$ & $19.6 \pm  1.8$ &	$21.6 \pm  1.4$ \\
VarSort &	$5.00 \pm 1.00$ &	$6.00 \pm 1.66$ &	$18.5 \pm 6.2$ &	$21.9 \pm 4.4$\\
\bottomrule
\end{tabular}
\end{sc}
  }
  \label{table-varsort}
  \end{table*}

    \begin{figure*}[t]
    \includegraphics[height=7.7cm]{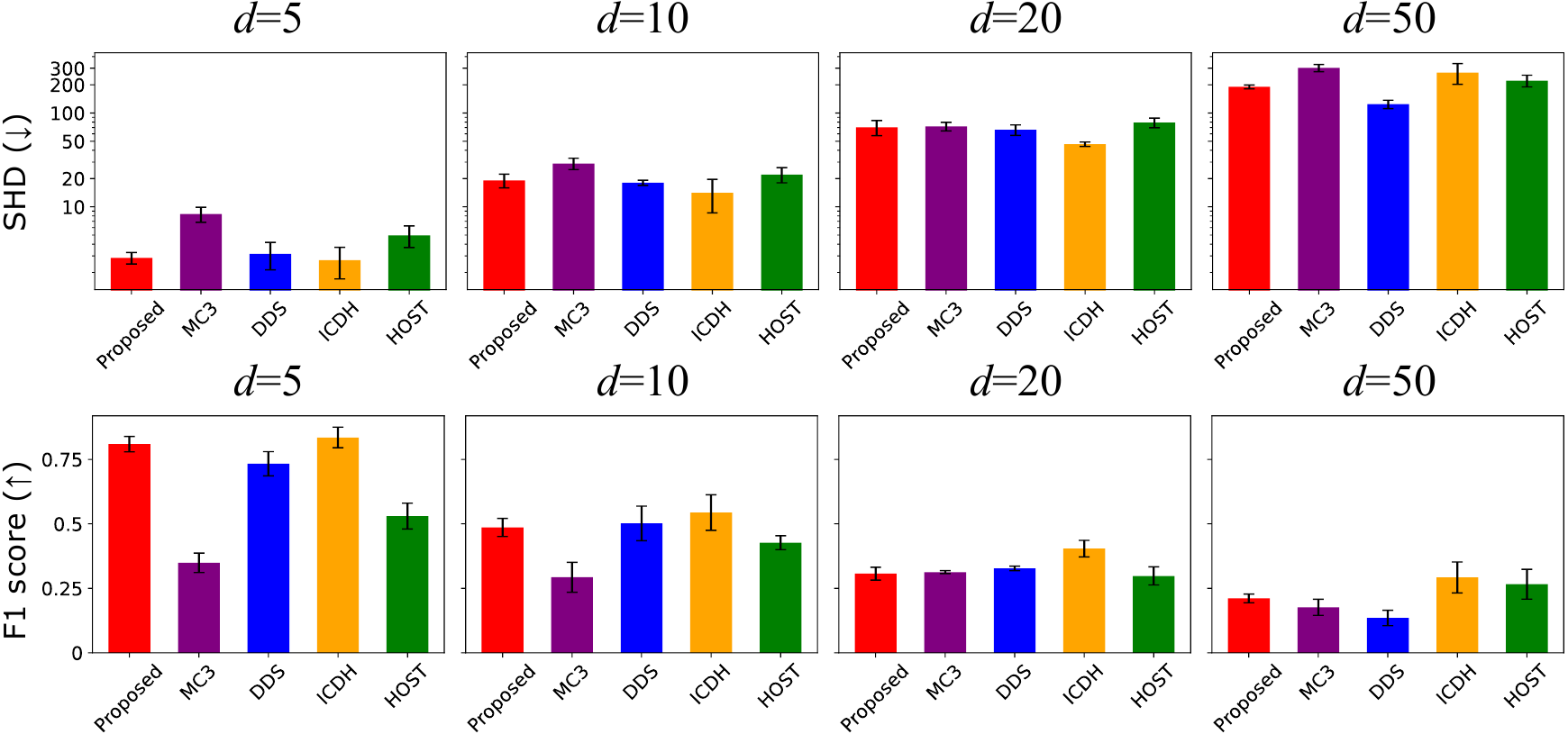}
  \centering 
  \caption{Moment-agnostic causal graph inference performance on nonlinear Gaussian ANM datasets ($n=500$) with number of nodes $d = 5, 10, 20, 50$:
  $\downarrow$ and $\uparrow$ denote "lower is better" and "higher is better".} 
  \label{fig-sameANMNN}
  \end{figure*}
  
  \begin{figure*}[t]
    \includegraphics[height=7.7cm]{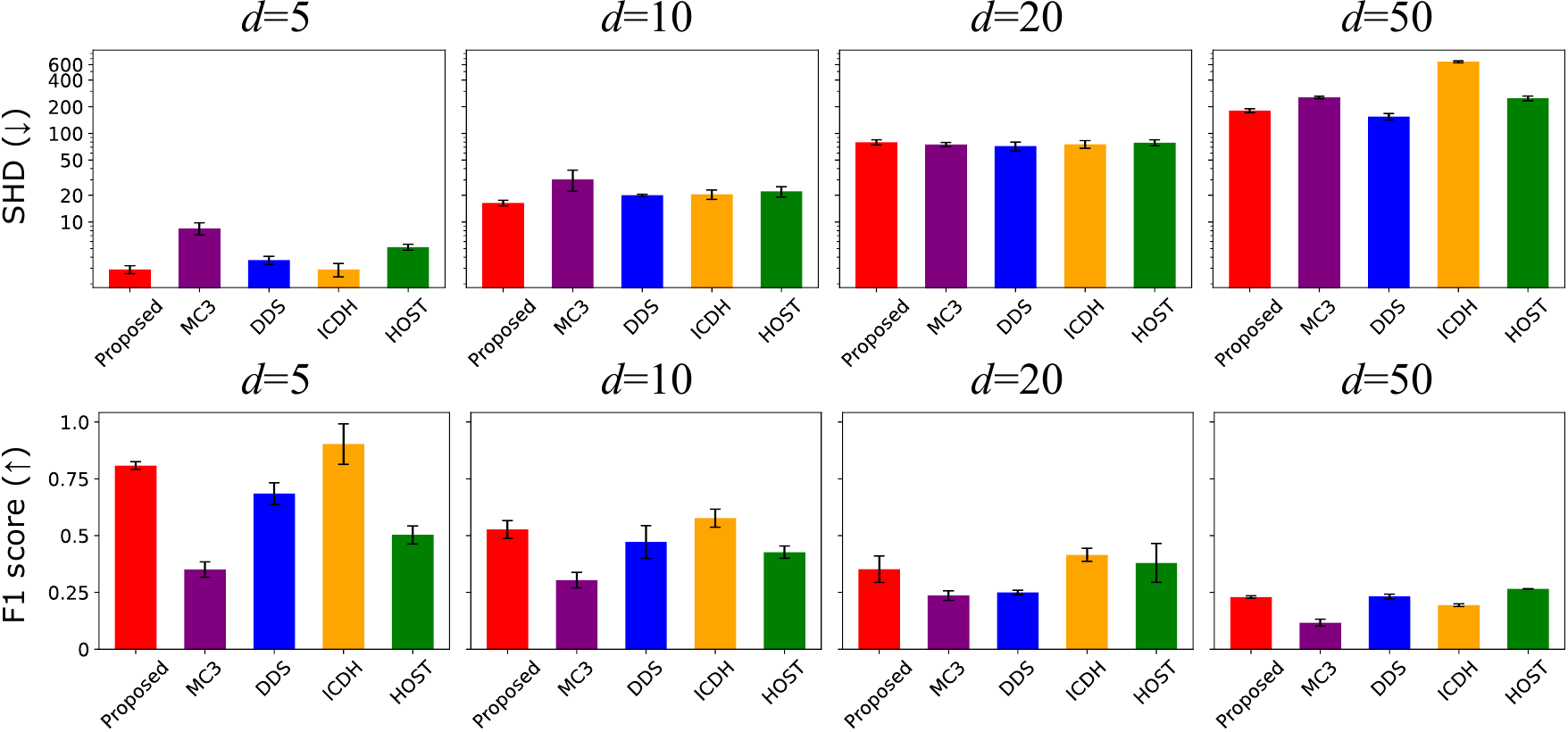}
  \centering 
  \caption{Moment-agnostic causal graph inference performance on nonlinear Gaussian HNM datasets ($n=500$) with number of nodes $d = 5, 10, 20, 50$:
  $\downarrow$ and $\uparrow$ denote "lower is better" and "higher is better".} 
  \label{fig-sameHNMNN}
  \end{figure*}

\end{document}